%% file: main.tex
\documentclass[11pt]{article}
\usepackage{times}
\usepackage{fullpage} % optional, makes margins decent

\input{math_commands.tex}

\usepackage{amsthm,amssymb}

\usepackage[round]{natbib}
\usepackage{hyperref}
\usepackage{cleveref}
\usepackage{url}
\usepackage{graphicx}
\usepackage{booktabs}
\usepackage{enumitem}

\DeclareMathOperator*{\vol}{vol}

\DeclareMathOperator*{\op}{op}

\theoremstyle{plain}
\newtheorem{theorem}{Theorem}
\newtheorem{lemma}[theorem]{Lemma}
\newtheorem{proposition}[theorem]{Proposition}
\newtheorem{example}[theorem]{Example}
\newtheorem{remark}[theorem]{Remark}
\theoremstyle{definition}

\newcommand{\x}{\bm{x}}                                 % a point is space
\newcommand{\y}{\bm{y}}                                 % vector of observations
\newcommand{\Sp}{\mathcal{S}}                           % search space
\newcommand{\Rp}{\R_{\geq 0}}                           % search space
\newcommand{\D}{\mathcal{D}}                            % finite dataset
\newcommand{\G}{\mathcal{G}}                            % symmetry group
\newcommand{\K}{{\bm {K}}}                            % symmetry group
\newcommand{\GP}{\mathcal{GP}}                          % Gaussian process
\newcommand{\kb}{k_{\mathrm{b}}}                        % base kernel
\newcommand{\kavg}{k_{\mathrm{avg}}}                  % avg kernel
\newcommand{\kmax}{k_{\mathrm{max}}}                    % max‑kernel
\newcommand{\kplus}{k_{+}}                              % PSD projection of max‑kernel
\newcommand{\act}[2]{{#1 #2}}                           % act{g}{x} = action of g on x
\newcommand\inner[2]{\left\langle #1, #2 \right\rangle} % dot product
\newcommand\blackbox{{f^\star}}
\newcommand\f{{f}}

\usepackage{subcaption}

\usepackage{microtype} % usually allowed; improves line breaks
\usepackage{enumitem}
\setlist{nosep,leftmargin=*} % tighter lists

\title{Symmetry-Aware Bayesian Optimization via Max Kernels}

\author{
Anthony Bardou\thanks{School of Computer and Communication Sciences, École polytechnique fédérale de Lausanne (EPFL), Station 14, CH-1015 Lausanne, Switzerland} \quad
Antoine Gonon\thanks{Institute of Mathematics, École polytechnique fédérale de Lausanne (EPFL), Station 8, CH-1015 Lausanne, Switzerland} \quad
Aryan Ahadinia\footnotemark[1] \quad
Patrick Thiran\footnotemark[1]
}

\begin{document}

\maketitle

\begin{abstract}
Bayesian Optimization (BO) is a powerful framework for optimizing noisy, expensive-to-evaluate black-box functions. When the objective exhibits invariances under a group action, exploiting these symmetries can substantially improve BO efficiency. While using maximum similarity across group orbits has long been considered in other domains, the fact that the max kernel is not positive semidefinite (PSD) has prevented its use in BO. In this work, we revisit this idea by considering a PSD projection of the max kernel. Compared to existing invariant (and non-invariant) kernels, we show it achieves significantly lower regret on both synthetic and real-world BO benchmarks, without increasing computational complexity.
\end{abstract}

\section{Introduction}\label{sec:intro}

Bayesian optimization (BO) tackles the maximization of a noisy, expensive-to-evaluate black-box $\blackbox:\Sp\subset\R^d\to\R$ using a Gaussian process (GP) surrogate. When prior knowledge says that $\blackbox$ is invariant on orbits $[\x] = \{\act{g}{\x} : g \in \G\}$ of a group $\G$, that is,
\[
\blackbox(\x)=\blackbox(\act{g}{\x})\quad(\forall g\in\G),
\]
embedding this invariance into the kernel can significantly improve sample efficiency. 
A classical and principled approach is to average a base kernel $\kb$ over group orbits 
(e.g., \citet{Kondor08GroupMLthesis,glielmo2017accurate,Brown24SampleEfficientBOInvariant}). Averaging yields a $\G$-invariant kernel with a clean RKHS interpretation 
(as discussed in \Cref{sec:background-invbo}), but as $|\G|$ grows it can dilute high-similarity alignments across orbits.

\textbf{From averaging to max-alignment.}
We revisit a simple idea—\emph{retain the strongest orbitwise alignment}—and adapt it to BO. Given a base kernel $\kb$ and a symmetry group $\G$, define
\begin{equation}\label{eq:kmax}
\kmax(\x,\x')=\max_{g,g'\in\G}\,\kb\!\big(\act{g}{\x},\,\act{g'}{\x'}\big),
\end{equation}
so that the similarity between $\x$ and $\x'$ is the best alignment over their orbits. While $\kmax$ is symmetric and $\G$-invariant, it is not positive semidefinite (PSD) in general and thus cannot serve directly as a GP covariance.

\textbf{A PSD, invariant surrogate via projection + Nyström.}
On a finite design set $\D$, we form the Gram matrix of $\kmax$ and project it onto the PSD cone (eigenvalue clipping), obtaining $\K_+$. Denoting by $\K_+^\dagger$ the Moore-Penrose pseudo-inverse of $\K_+$, we then define the $\G$-invariant, PSD kernel
\begin{equation}
\label{def:kplusD-intro}
\kplus^{(\D)}(\x,\x')\;=\;\kmax(\x,\D)\,\K_+^{\dagger}\,\kmax(\D,\x').
\end{equation}
Equivalently, $\kplus^{(\D)}(\x,\x')=\phi(\x)^\top\phi(\x')$ with features $\phi(\x)=\K_+^{\dagger/2}\,\kmax(\D,\x)$, which makes positive semidefiniteness immediate. By construction, $\kplus^{(\D)}$ (i) coincides with $\kmax$ on $\D$ whenever $\kmax$ is already PSD, and (ii) has per-iteration asymptotic cost comparable to orbit-averaged kernels; details in \Cref{sec:max_ker-kplus}.

\textbf{Why can max-alignment help?}
Averaging mixes all orbit pairings and can shrink contrasts as $|\G|$ increases. In contrast, \eqref{eq:kmax} preserves high-contrast alignments that drive exploration, while the projection step~\eqref{def:kplusD-intro} produces a valid GP kernel without introducing new algorithmic complexity (BO iterations already perform a Singular Value Decomposition~(SVD) of the Gram matrix for GP inference, so the extra-computation of $\K_+^{\dagger}$ does not change the asymptotic cost as we will see later in \Cref{tab:complexities}).

\textbf{Empirics and spectra.}
Across synthetic benchmarks with finite and continuous groups and a wireless-network design task, we show that $\kplus^{(\D)}$ consistently attains lower cumulative and simple regret than both the base kernel and the orbit-averaged alternative, with gains increasing with $|\G|$. Our spectral analyses reveal that $\kplus^{(\D)}$ does not necessarily enjoy faster eigendecay than averaging-based kernels; thus, eigendecay-based regret bounds would predict similar or weaker rates, yet we observe the opposite in practice. We hypothesize that search-geometry effects (e.g., preserving high-contrast orbit alignments),  approximation hardness and misspecification are key factors to take into account to fill this gap between theory and practice; see \Cref{sec:eigendecay}.

\textbf{Summary of the contributions.} 
We propose $\kmax$ as a \emph{max-alignment} route to $\G$-invariance, turn it into a valid GP kernel for BO via PSD projection and Nyström, and show $\kplus^{(\D)}$ is $\G$-invariant, equals $\kmax$ on $\D$ when $\kmax$ is PSD, and matches the asymptotic cost of orbit-averaged kernels (\Cref{sec:max_ker}). We demonstrate consistent BO gains over orbit averaging across BO benchmarks (\Cref{sec:xps}), and we analyze why eigendecay alone does not explain these gains (\Cref{sec:eigendecay}).

\section{Background} \label{sec:background}

\subsection{Bayesian Optimization in a Nutshell}
\label{sec:bo_background}

\textbf{Problem.}
We seek to maximize an expensive-to-evaluate, black-box objective $\blackbox:\Sp\!\to\!\R$ under the assumption that $\blackbox$ is in the RKHS $\mathcal{H}_k$ of a kernel $k:\Sp\times\Sp\to \R$. 
Each query $\x\in\Sp$ returns a noisy observation
$y=\blackbox(\x)+\varepsilon$, where $\varepsilon\sim\mathcal N(0,\sigma_0^2)$. 
Let $\mathcal Z_t = \{(\x_i,y_i)\}_{i=1}^t$ denote the dataset after $t$ evaluations, and write
$\D_t=(\x_1,\ldots,\x_t)$ and $\bm y_t=(y_1,\ldots,y_t)^\top$. 

\textbf{Surrogate model: the GP prior.}
BO maintains a probabilistic surrogate $\f$ over functions in $\mathcal{H}_k$ to guide sampling of new queries $\x\in\Sp$ with the goal of converging to $\argmax_{x\in\Sp} \blackbox(\x)$. 
A common choice is a zero-mean Gaussian process (GP)~\citep{williams2006gaussian},
\[
\f \sim \GP(0,k), 
\]
Conditionally on $\mathcal Z_t$, the posterior $\f\mid \mathcal Z_t$ is still a GP with posterior mean and covariance
\begin{align}
\mu_t(\x) &= k(\x,\D_t)\,\big(\K_t+\sigma_0^2 \bm I_t\big)^{-1}\bm y_t, \label{eq:post_mean}\\
\Cov_t(\x,\x') &= k(\x,\x') - k(\x,\D_t)\,\big(\K_t+\sigma_0^2 \bm I_t\big)^{-1}k(\D_t,\x'), \label{eq:post_cov}
\end{align}
where $\K_t=k(\D_t,\D_t)\in\R^{t\times t}$, $\bm I_t$ is the $t\times t$ identity, and
$k(\x,\D_t)=[k(\x,\x_1),\ldots,k(\x,\x_t)]$.

\textbf{BO Iteration.} 
At step $t$, BO trades off \emph{exploration} (learning $\blackbox$) and \emph{exploitation}
(sampling near current optima) via an acquisition function $\alpha_t:\Sp\to\R$ computed from
$(\mu_t,\Cov_t)$ (e.g., GP-UCB~\citep{srinivas-information:2012} or
Expected Improvement~\citep{jones1998efficient}). The next point is
\[
\x_{t+1}\in\argmax_{\x\in\Sp}\,\alpha_t(\x),\qquad
y_{t+1}=\blackbox(\x_{t+1})+\varepsilon_{t+1}.
\]
\textbf{Measuring performance with regret.} We follow the common practice in BO: for experiments where $\blackbox$ is known, we measure the regret on the deterministic $\blackbox\in\mathcal{H}_k$, and when discussing theoretical regret bounds we refer to the regret on $\f\sim\GP(0,k)$~\citep{garnett2023bayesian}. In both cases, for $h=\f$ or $h=\blackbox$, the \emph{instantaneous regret} at timestep $t$ is $r_t=\max_{\x\in\Sp} h(\x)-h(\x_t)$, the
\emph{cumulative regret} at horizon $T$ is $R_T=\sum_{t=1}^T r_t$, and the
\emph{simple regret} is $s_T=\max_{\x\in\Sp} h(\x)-\max_{1\le t\le T} h(\x_t)$. A BO algorithm with a sublinear regret (i.e., $R_T \in o(T)$) is called \emph{no-regret} and offers asymptotic global optimization guarantees on $\blackbox$. Most standard cumulative regret upper bounds are established in terms of the eigendecay of the operator spectrum of the kernel $k$~\citep{srinivas-information:2012, valko2013finite, scarlett2017lower, whitehouse2023sublinear}.

\subsection{Invariance in Bayesian Optimization} \label{sec:background-invbo}

In many applications, the objective function $\blackbox$ is invariant under the action of a known symmetry group $\G$ on $\Sp$, i.e., $\blackbox(\x) = \blackbox(\act{g}{\x})$ for all $g \in \G$. When such invariances are ignored, BO algorithms may waste evaluations by treating all points within the same $|\G|$-orbit as distinct. Given a non-invariant base kernel $\kb$ and an arbitrary symmetry group $\G$, both provided by the user, this section reviews existing strategies for incorporating group invariance into BO and positions our contribution within this literature.

\textbf{Data augmentation.} A popular way to enforce symmetry is to expand the dataset $\mathcal Z$ itself, as it is often done in computer vision~\citep{krizhevsky2012imagenet}. For each acquired observation $(\x_t, y_t)$, one augments $\mathcal Z$ with all transformed copies $\{(\act{g}{\x_t}, y_t)\}_{g \in \G}$, while leaving the base kernel $\kb$ unchanged. However, since BO scales as $\mathcal{O}(|\mathcal Z|^3)$, this approach quickly becomes computationally prohibitive and is inapplicable to continuous symmetry groups.

\textbf{Search space restriction.} A second strategy is to restrict the optimization domain to the smallest subset $\Sp_\G \subseteq \Sp$ such that $\bigcup_{g \in \G} g\Sp_\G = \Sp$ (e.g., \citet{BAIRD2023112134}). For instance, if $\Sp = [-1,1]^2$ and $\G$ is the group of planar rotations by angle $\pi/2$, then it suffices to optimize over $\Sp_\G = [0,1]^2$ while keeping the base kernel $\kb$ unchanged. In general, however, identifying a suitable fundamental domain $\Sp_\G$ can be challenging, and enforcing optimization within it may be impractical.

\textbf{Invariant kernels.} A principled way to incorporate prior $\G$-invariance of $\blackbox$ is to consider a $\G$-invariant GP prior $\f$, i.e., a GP whose sample paths $\x\in\Sp \mapsto \f(\x,\omega)$ obtained by fixing one outcome $\omega$ in the probability space are themselves invariant under $\G$. \citet{ginsbourger2012argumentwise} established that such GPs necessarily admit a $\G$-invariant covariance function\footnote{Up to modification, i.e., there is another GP $\f'$ such that for every $x\in \Sp$, $\mathbb{P}(\f(\x) = \f'(\x))=1$ and $\f'$ has invariant paths and invariant covariance, see Property 3.3 in \citet{ginsbourger2012argumentwise}.}, meaning $k(g\x, g'\x') = k(\x, \x')$ for all $\x,\x' \in \Sp$ and $g,g' \in \G$. The central question then becomes: how can one construct an invariant kernel $k$ from an arbitrary base kernel $\kb$ and symmetry group $\G$? An elegant solution, dating back to \citet{Kondor08GroupMLthesis} and recently advocated for BO  by~\citet{Brown24SampleEfficientBOInvariant}, is to average $\kb$ over $\G$-orbits:
\begin{equation} \label{eq:kavg}
    \kavg(\x, \x') = \frac{1}{|\G|^2} \sum_{g, g' \in \G} \kb(g\x, g'\x').
\end{equation}
This construction is not only guaranteed to be $\G$-invariant, but also admits a clean functional interpretation: if $\mathcal{H}_{\kb}$ and $\mathcal{H}_{\kavg}$ denote the RKHS induced by $\kb$ and $\kavg$ respectively, then $\mathcal{H}_{\kavg}$ coincides exactly with the subspace of $\G$-invariant functions in $\mathcal{H}_{\kb}$ (Theorem~4.4.3 in \citet{Kondor08GroupMLthesis}). Consequently, $\kavg$ (up to normalization) has gained popularity as the standard off-the-shelf kernel for BO in symmetric settings~\citep{glielmo2017accurate, kim2021bayesian, Brown24SampleEfficientBOInvariant}.

A complementary idea in kernel methods is to retain the \emph{best} latent alignment between two orbits via a maximum, as in convolution/best-match kernels for structured data~\citep{gartner2003survey, vishwanathan2003fast} and follow-up work across domains~\citep{frohlich2005optimal, zhang2010maximum, curtin2013fast}. Max-alignment kernels, however, are not PSD in general, leading to indefinite Gram matrices. This has motivated two families of remedies: (i) explicit Kreĭn-space formulations~\citep{ong2004learning, oglic2018learning}, and (ii) simple PSD corrections such as eigenvalue clipping/flipping in SVMs~\citep{luss2007support, chen2009learning}, which are empirically effective. 

\textbf{Our adaptation to BO.}
Guided by the above, we adopt the max-alignment view for BO. To ensure positive definiteness, we project $\kmax$ (see~\eqref{eq:kmax}) onto a PSD kernel $\kplus^{(\D)}$, which coincides with $\kmax$ whenever the latter is already PSD. This preserves the sharp, high-contrast orbit alignments of $\kmax$ while ensuring compatibility with the BO framework and it keeps per-iteration BO complexity on par with orbit-averaged kernels (see complexity details later in \Cref{sec:background-invbo}). In our experiments, $\kplus^{(\D)}$ better reflects the intended symmetries of standard synthetic objectives and achieves substantially lower cumulative regret; interestingly, these empirical gains are not mirrored by existing eigendecay-based upper bounds, a point we return to in \Cref{sec:eigendecay}.

\section{The Max Kernel} \label{sec:max_ker}
We have introduced the max-alignment kernel $\kmax$ and its PSD surrogate $\kplus^{(\D)}$ in \eqref{def:kplusD-intro}. This section explains \emph{why} $\kmax$ is a natural $\G$-invariant covariance, clarifies how it differs from orbit averaging through examples, and records the practical PSD construction we use in BO.

\subsection{Motivation: $\kmax$ as a valid covariance} \label{sec:max_ker-motivation}

A natural way to motivate $\kmax$ is to exhibit $\G$-invariant GPs whose covariance equals $\kmax$.

\textbf{Construction.}
Let $h\sim\GP(0,\kb)$ with an isotropic base kernel $\kb(\x,\x')=\kappa(\|\x-\x'\|_2)$ with $\kappa$ nonincreasing (e.g., popular ones such as RBF, Matérn).
Consider a map $\phi_\G$ such that (i) $\phi_\G(\x)=\phi_\G(\act{g}{\x})$ for all $g\in\G$ and
(ii) $\|\phi_\G(\x)-\phi_\G(\x')\|_2=\min_{g,g'}\|\act{g}{\x}-\act{g'}{\x'}\|_2$.
Define $\f(\x)=h(\phi_\G(\x))$. Then $\f$ is $\G$-invariant and:

\begin{proposition}\label{prop:kmax_is_adequate}
Under the construction above, $\f\sim\GP(0,\kmax)$ with $\kmax$ given by~\eqref{eq:kmax}.
\end{proposition}
\noindent
\emph{Proof sketch, details in \Cref{app:kmax_proofs}.} $\Cov(\f(\x),\f(\x'))\overset{\text{def} \f}{=}\kb(\phi_\G(\x),\phi_\G(\x'))\overset{\text{(ii)}}{=} \kappa(\min_{g,g'}\|\act{g}{\x}-\act{g'}{\x'}\|_2)$, and
monotonicity of $\kappa$ converts the min-distance into $\max_{g,g'}\kb(\act{g}{\x},\act{g'}{\x'})$.\qed

This shows that $\kmax$ naturally arises as the covariance of valid $\G$-invariant GPs.  
In contrast, the common approach to invariance in BO is to build $\kavg$ by averaging a base kernel as in~\eqref{eq:kavg}. 
But averaging and maximization induce fundamentally different geometries:

\begin{lemma}\label{prop:kavg_does_not_fit}
For any base kernel $\kb$ and any (double) orbit $\mathcal O(\x,\x'):=\{(\act{g}{\x}, \act{g'}{\x'}), g,g'\in\G\}$, 
$\kavg=\kmax$ on $\mathcal O(\x,\x')$ if and only if $\kb=\kmax$ on that orbit. 
\end{lemma}
% \begin{proof}
% An average reaches the maximum only when every term is maximal. 
% \end{proof}
Indeed, an average reaches the maximum only when every term is maximal. Thus $\kavg$ can never reproduce the geometry of $\kmax$, except in the degenerate case where the base kernel is already $\kmax$, making averaging redundant. One might wonder whether this limitation of $\kavg$ could be circumvented by building it from a \emph{different} base kernel than the one used for $\kmax$. In \Cref{app:kavg_vs_kmax} we show that, under mild assumptions satisfied by standard kernels (upper-bounded by $1$, with equality $k(\x,\x)=1$ along the diagonal), $\kavg$ and $\kmax$ can coincide only in the trivial case where the base kernel of $\kavg$ is already invariant when its arguments belong to the same orbit. 
%partially $\G$-invariant.
Thus, even in this more general setting, averaging does not reproduce the geometry of maximization (except if the base kernel already had invariances).

To make this contrast concrete, we now examine a simple example 
(radial invariance with an RBF base kernel) where $\kmax$ and $\kavg$ 
can be computed in closed form. 

\begin{example}[Radial invariance with $\kmax$]
\label{ex:radial}
Let $\G$ be the group of planar rotations and $\kb(\x,\x') = \exp\!\left(-\|\x-\x'\|_2^2 / 2l^2\right)$ be an RBF kernel. With $\phi_\G(\x)=\|\x\|_2$,
\[
\kmax(\x,\x')=\exp\!\left(-(\|\x\|_2-\|\x'\|_2)^2/2l^2\right),\quad
\kavg(\x,\x')=\exp\!\Big(-\tfrac{\|\x\|_2^2+\|\x'\|_2^2}{2l^2}\Big)I_0\!\Big(\tfrac{\|\x\|_2\|\x'\|_2}{l^2}\Big),
\]
with $I_0$ the modified Bessel function (derivation in \Cref{app:radial_kavg}). As illustrated in \Cref{fig:motivating_example}, the two kernels $\kmax$ and $\kavg$ induce qualitatively different similarity structures. 
By construction, $\kmax$ assigns large similarity whenever $\|\x\|_2 \approx \|\x'\|_2$. If $\|\x\|_2 = \|\x'\|_2$, the function $\blackbox$ satisfies $\blackbox(\x)=\blackbox(\x')$ since it is invariant under rotations, and $\kmax$ exactly recovers this invariance by assigning maximal similarity $\kmax(\x,\x')=1$. 
In contrast, $\kavg$ only approximates this behavior: its iso-similarity curves as a function of $(\|\x\|_2,\|\x'\|_2)$ correspond to distorted balls, and two points with identical norms may be ranked as highly dissimilar (see the diagonal $\|\x\|_2=\|\x'\|_2$ of the right plot in \Cref{fig:motivating_example}). 
This mismatch highlights that while both constructions enforce rotation invariance, only $\kmax$ preserves the correct notion of similarity.
\end{example}

\begin{figure}[t]
    \centering
    \includegraphics[height=3.5cm]{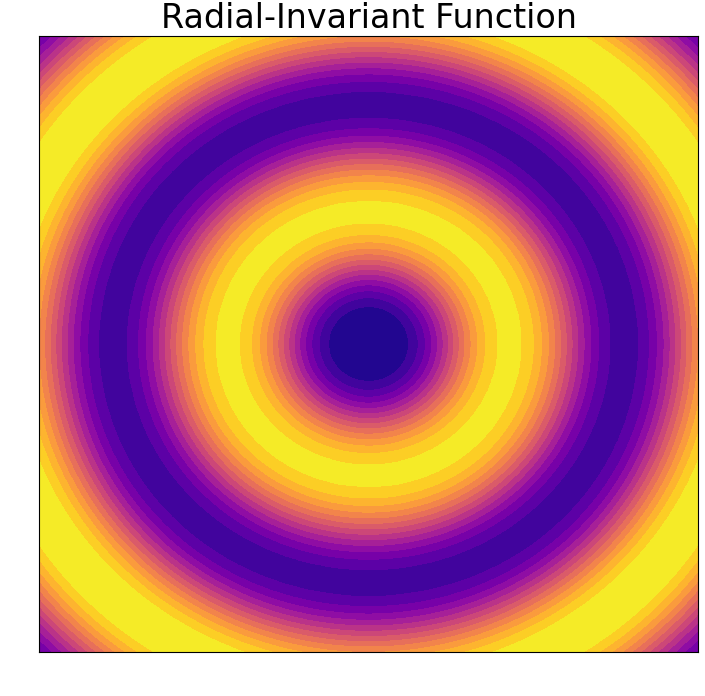} \quad
    \includegraphics[height=3.5cm]{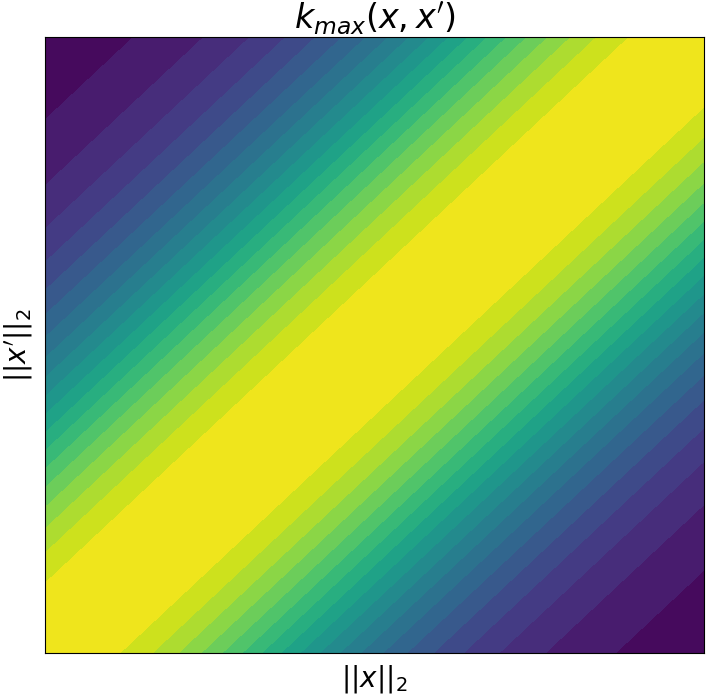} \quad
    \includegraphics[height=3.5cm]{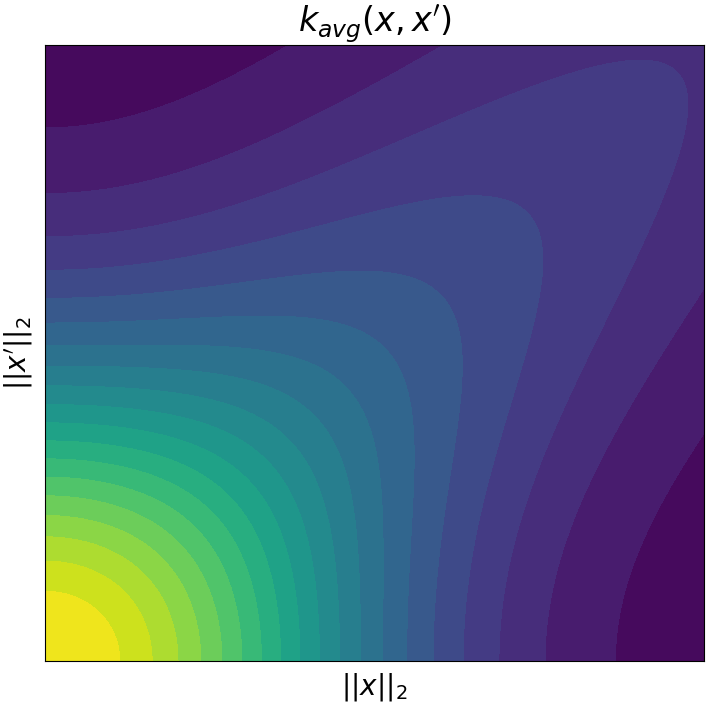}
    \caption{(Left)~A two-dimensional function $\blackbox(\x)$ invariant under planar rotations (see~\eqref{def:radial}): if $\|\x\|_2=\|\x'\|_2$, then $\blackbox(\x)=\blackbox(\x')$. (Center/Right)~Rotation-invariant kernels derived from an RBF base kernel (lengthscale $1/2$), visualized as a function of $(\|\x\|_2,\|\x'\|_2)$. $\kmax$ (center) captures the correct invariance, while $\kavg$ (right) only approximates it.}
    \label{fig:motivating_example}
\end{figure}

\subsection{A PSD Extension of \texorpdfstring{$\kmax$}{kmax}: What We Use in Practice}
\label{sec:max_ker-kplus}
%
% Because $\kmax$ is not PSD in general (\Cref{app:kmax_not_psd}), we apply...
Because $\kmax$ is not PSD in general, we apply a standard projection step on the finite design set $\D=\{\x_1,\dots,\x_n\}$. Let $\K=\kmax(\D,\D)$ with eigendecomposition $\K=\bm Q\bm \Lambda \bm Q^\top$ and define\footnote{$\K_+$ does not depend on the choice of the eigendecomposition, see \Cref{lem:finite-psd-projection} in the appendix.} (with the max applied elementwise)
\begin{equation}
\label{eq:def-K-plus}
\K_+ \;=\; \bm Q\,\max(0,\bm \Lambda)\,\bm Q^\top.
\end{equation}
We then use the Nyström extension\footnote{It indeed extends $\K_+$ since $\kplus^{(\D)}(\x_i,\x_j)
\;=\; \K_{i,:}\,\K_+^{\dagger}\,\K_{:,j}
\;=\; (\K \K_+^{\dagger} \K)_{ij}
% \;=\; Q\,\Lambda\,\Lambda_+^{\dagger}\,\Lambda\,Q^\top
\;=\; (\K_+)_{ij}$.
}~\citep{williams2000using} to evaluate cross-covariances with new points, yielding the PSD, $\G$-invariant surrogate $\kplus^{(\D)}$ given in \eqref{def:kplusD-intro} and that we reproduce here:
\begin{equation}
\label{eq:nystrom-kplus}
\kplus^{(\D)}(\x,\x')
\;:=\; \kmax(\x,\D)\,\K_+^{\dagger}\,\kmax(\D,\x') .
\end{equation}  

\begin{table}[t]
\centering
\caption{Complexity per BO iteration. 
Here $|G|^{\!*}$ denotes either $|G|$ or $|G|^2$ depending on whether 
the orbit terms reduce to a single sum (when $\kb(\act{g}{\x},\x')$ suffices) 
or require a double sum over $(g,g')$; $m$ is the number of candidate points 
used in acquisition optimization.}
% \small
\resizebox{\linewidth}{!}{%
\begin{tabular}{lccc}
\toprule
 & Base kernel $\kb$ & Averaged $\kavg$ & Projected $\kplus^{(\D)}$ \\
\midrule
Gram matrix ($n\times n$)
 & $\mathcal{O}(n^2)$ 
 & $\mathcal{O}(n^2|G|^{\!*})$ 
 & $\mathcal{O}(n^2|G|^{\!*})$ \\
SVD / inversion
 & $\mathcal{O}(n^3)$ 
 & $\mathcal{O}(n^3)$ 
 & $\mathcal{O}(n^3)$ \\
PSD projection
 & -- & -- & $\mathcal{O}(n^3)$\footnotemark \\
Per-query evaluation% $(\x,\x')$ 
 & $\mathcal{O}(1)$ 
 & $\mathcal{O}(|G|^{\!*})$ 
 & $\mathcal{O}(n|G|^{\!*})$ \\
\midrule
\textbf{BO iteration} 
 & {$\mathcal{O}(m+n^2+n^3)$} 
 & {$\mathcal{O}((m+n^2)|G|^{\!*}+n^3)$}
 & {$\mathcal{O}((mn+n^2)|G|^{\!*}+n^3)$} \\
\bottomrule
\end{tabular}
}
\label{tab:complexities}
\end{table}
\footnotetext{One SVD of $\K$ suffices to obtain both $\K_+$ and $\K_+^\dagger$, 
so the extra PSD projection does not increase asymptotic cost.}

\textbf{Key properties of $\kplus^{(D)}$:}
\begin{itemize}[noitemsep, topsep=0pt, leftmargin=*]
\item \emph{PSD \& invariance.} 
$\kplus^{(\D)}$ is PSD and inherits argumentwise $\G$-invariance\footnote{$\kmax(\act{g}{\x},\x')=\kmax(\x,\x')$ implies $\kmax(\act{g}{\x},\D)=\kmax(\x,\D)$, hence invariance of $\kplus^{(\D)}$.} of $\kmax$.
\item \emph{Consistency with $\kmax$.} If $\K\succeq 0$, then $\K_+=\K$ and $\kplus^{(\D)}$ agrees with $\kmax$ on $\D\times\D$.
\item \emph{Cost.} Each BO iteration involves (i) building the Gram matrix on $\D$, 
(ii) inverting the Gram matrix to build the acquisition function, and 
(iii) $m$ kernel evaluations when optimizing the acquisition function.  
Step (ii) has the same cost as the SVD of $\K$ needed to compute both $\K_+$ and $\K_+^\dagger$, which makes $\kplus^{(\D)}$ having the same asymptotic per-iteration cost as $\kavg$; its per-query evaluations are more expensive, 
but this difference is negligible as long as we keep $m\lesssim n$. A concise complexity summary is provided in \Cref{tab:complexities}. 
\item \emph{Regularity.} For finite groups, $\kmax$ is a max of finitely many smooth maps and is almost everywhere~(a.e.)~differentiable; the Nyström extension preserves a.e.\ differentiability in each argument. For continuous groups, smoothness can sometimes be obtained via closed-form formulas (e.g., as in \Cref{ex:radial}). 
\end{itemize}

We now illustrate the behavior of $\kplus^{(\D)}$ versus $\kavg$ (in this situation, $\kmax$ is not PSD and the projection step is indeed needed to restore positive semidefiniteness).

\begin{example}[Ackley function with $\kplus$]
\label{ex:ackley}
\Cref{fig:ackley_kernels} compares $\kplus^{(\D)}$ and $\kavg$ 
on the one-dimensional Ackley function (see~\eqref{def:ackley}). 
The projected kernel $\kplus^{(\D)}$ preserves the expected pairwise 
symmetries (invariance along $x=y$ and $x=-y$) and spreads mass 
more evenly across the symmetric regions, whereas $\kavg$ 
concentrates covariance mostly near the origin. 
Thus, $\kplus^{(\D)}$ better reflects the symmetry geometry 
of the problem, echoing the qualitative difference observed 
in \Cref{ex:radial}.
\end{example}

\begin{figure}[h]
\centering
\includegraphics[height=3.5cm]{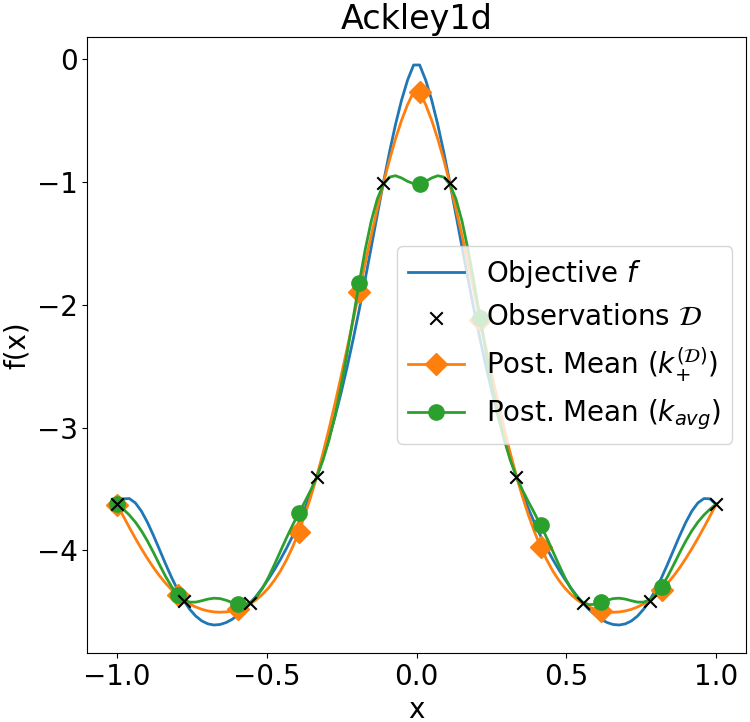} \quad
\includegraphics[height=3.5cm]{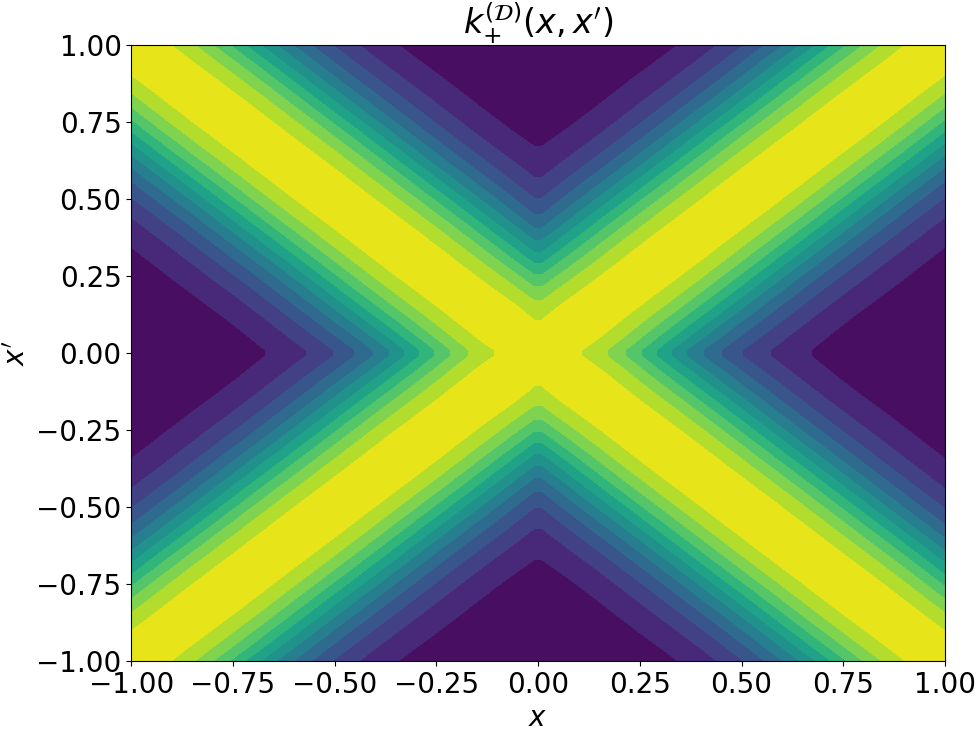}\quad
\includegraphics[height=3.5cm]{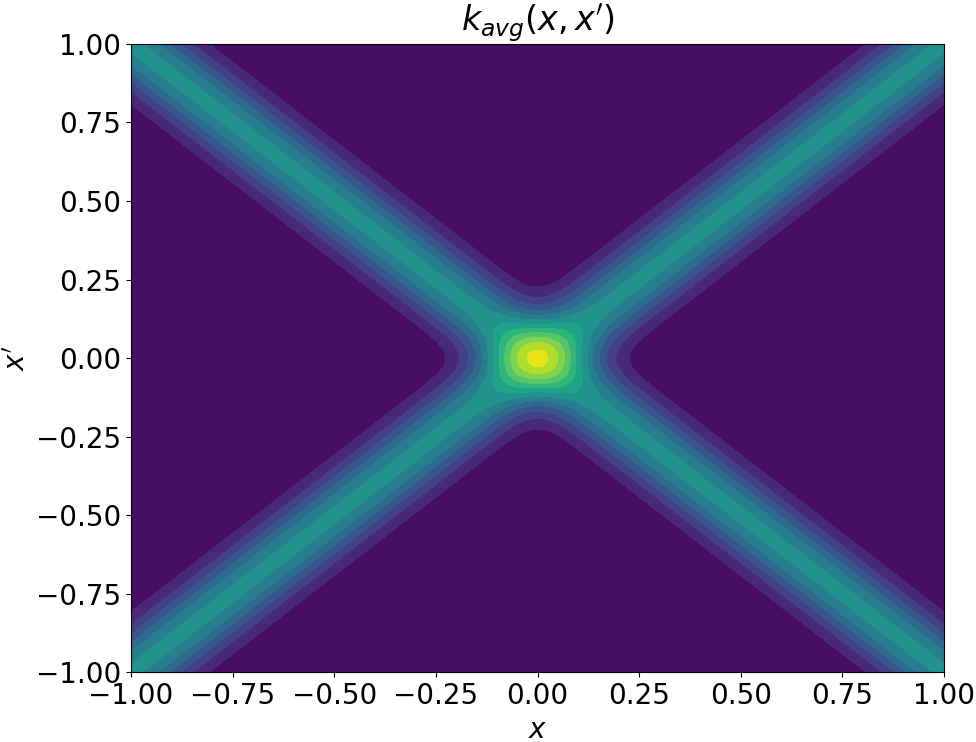}
\caption{(Left)~One-dimensional Ackley function $\blackbox$ (see \eqref{def:ackley}), invariant up to sign flips, and GP posterior means $\mu_t(\x)$ as in \eqref{eq:post_mean} for $\kplus^{(\D)}$~(orange diamond) and $\kavg$~(green circles) built from $\D$~(black crosses).  (Center)~Covariance structure induced by $\kplus^{(\D)}$. (Right)~Covariance structure induced by $\kavg$. Both kernels are invariant to reflections across $x=y$ and $x=-y$, but $\kavg$ concentrates covariance near $0$, 
while $\kplus^{(\D)}$ better reflects the underlying symmetry geometry. Consequently, the GP posterior mean induced by $\kplus^{(\D)}$ is the best at fitting the objective (left).}
\label{fig:ackley_kernels}
\end{figure}

\textbf{Beyond the finite view (details in \Cref{app:kplus}).} 
The PSD projection with Nyström in \Cref{eq:nystrom-kplus} is a practical, 
data-dependent construction. It can be seen as the finite-sample face of a 
broader, intrinsic definition that does not depend on $\D$. Since $\kmax$ is 
symmetric, it admits a spectral decomposition 
$\kmax(\x,\x')=\sum_i \lambda_i \phi_i(\x) \phi_i(\x')$ in $L^2$, and we can always define (a.e.)
\[
\kplus(\x,\x') := \sum_i \max(0,\lambda_i)\,\phi_i(\x) \phi_i(\x') ,
\]
with $\kplus=\kmax$ whenever $\kmax$ is already PSD. On finite domains, this 
precisely reduces to the matrix PSD projection in \eqref{eq:def-K-plus}. In
\Cref{app:kplus} we formalize the infinite-domain construction via integral
operators, prove that $\kplus$ is $\G$-invariant, and show that the finite
projection $+$ Nyström in \eqref{eq:nystrom-kplus} converges to $\kplus$ at the
spectral (Hilbert-Schmidt) level under iid sampling (\Cref{app:kplus-approx}).

\textbf{Takeaway.} $\kmax$ is the exact covariance of a natural class of $\G$-invariant GPs and induces a search geometry that preserves high-contrast orbit alignments (\Cref{ex:radial,ex:ackley}). The PSD projection + Nyström step yields a valid GP kernel $\kplus^{(\D)}$ without introducing extra asymptotic complexity. We now measure its practical impact in \Cref{sec:xps}.

\section{Experiments}
\label{sec:xps}

We evaluate $\kplus^{(\D)}$ against two baselines: (i) the off-the-shelf kernel $\kb$ (no symmetry handling), and (ii) the orbit-averaged kernel $\kavg$~\citep{Brown24SampleEfficientBOInvariant}. Benchmarks include standard synthetic objectives and a real-world wireless design task with known invariances. We ask: \emph{(Q1) Does $\kplus^{(\D)}$ reduce simple/cumulative regret vs.\ $\kavg$?} and \emph{(Q2) How does performance scale with the size of the symmetry group and dimension?} Experimental details are in \Cref{app:benchmarks}.

\begin{table}[t]
    \centering
    \caption{Performance of $\kb$, $\kavg$, and $\kplus^{(\D)}$ across benchmarks. For each kernel $k\in\{\kb,\kavg,\kplus^{(\D)}\}$ we report $m \pm s_{\text{err}}$, where $m$ is the empirical mean over 10 seeds (lower is better) and $s_{\text{err}}$ is the empirical standard error. Best mean is \textbf{bold}; means $m$ whose 95\% confidence interval ($m \pm 1.96 s_{\text{err}}$) confidence interval overlap with the best are \underline{underlined}. Performance is measured by cumulative regret on synthetic benchmarks and by negated simple reward on real-world experiments.}
    \label{tab:results}
    \resizebox{\linewidth}{!}{%
    \begin{tabular}{lcccc}
        \toprule
        \textbf{Benchmark} & $|\G|$ & $\kb$ & $\kavg$ & $\kplus^{(\D)}$\\
        \midrule
        \textit{Synthetic (Cumulative Reg.)} & & & &\\
        Ackley2d & $8$ & $382.7 \pm 5.7$ & \underline{$128.2 \pm 10.4$} & $\mathbf{126.4 \pm 3.6}$\\
        Griewank6d & $64$ & $3840.3 \pm 177.7$ & \underline{$3067.4 \pm 841.9$} & $\mathbf{1832.6 \pm 146.3}$\\
        Rastrigin5d & $3,840$ & $3568.5 \pm 91.3$ & \underline{$1583.5 \pm 341.9$} & $\mathbf{813.4 \pm 70.6}$\\
        Radial2d & $\infty$ & $388.6 \pm 20.3$ & $480.9 \pm 76.4$ & $\mathbf{199.7 \pm 11.6}$\\
        Scaling2d & $\infty$ & \underline{$1820.6 \pm 1135.4$} & $3361.8 \pm 742.9$ & $\mathbf{25.4 \pm 6.4}$\\
        \midrule
        \textit{Real-World (Neg. Simple Rew.)} & & & &\\
        WLAN8d & $24$ & $-65.0 \pm 3.2$ & $-51.8 \pm 1.7$ & $\mathbf{-74.4 \pm 0.7}$\\
        \bottomrule
    \end{tabular}
    }
\end{table}

\begin{figure}[t]
    \centering
    \includegraphics[height=3.2cm]{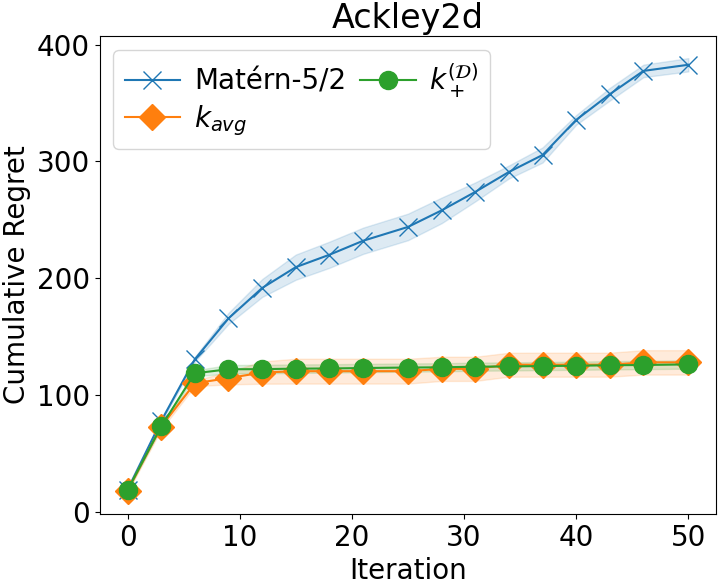} \quad
    \includegraphics[height=3.2cm]{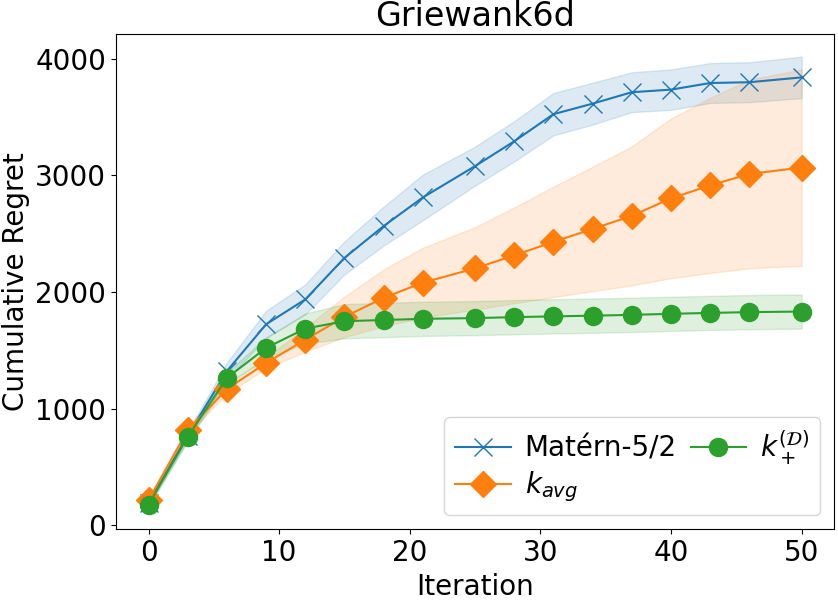} \quad
    \includegraphics[height=3.2cm]{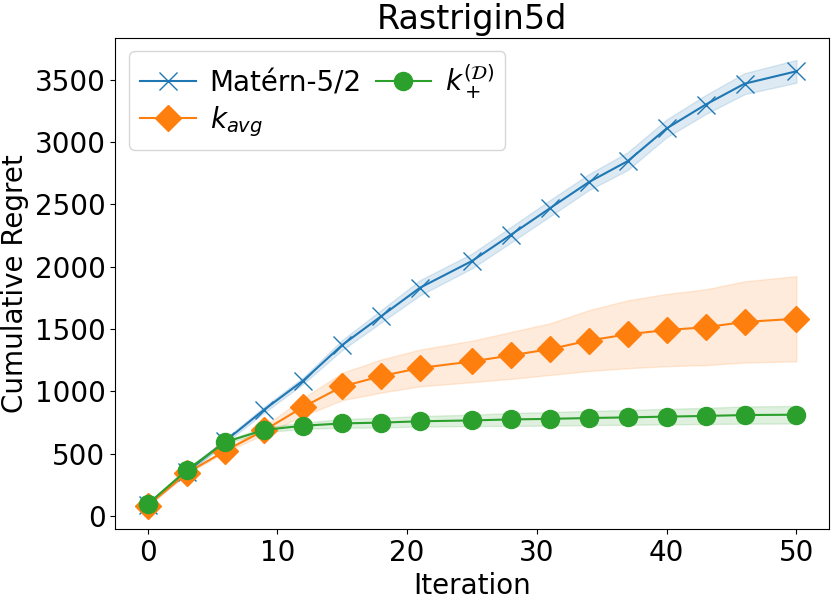} \quad
    \includegraphics[height=3.2cm]{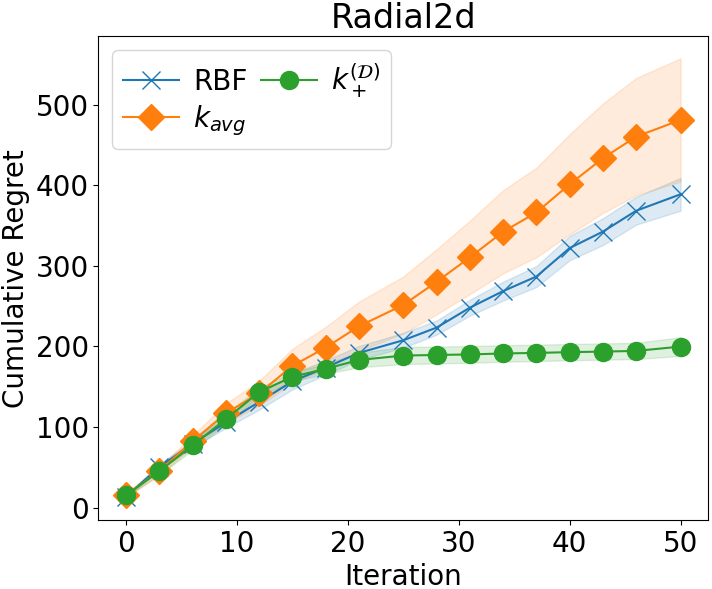} \quad
    \includegraphics[height=3.2cm]{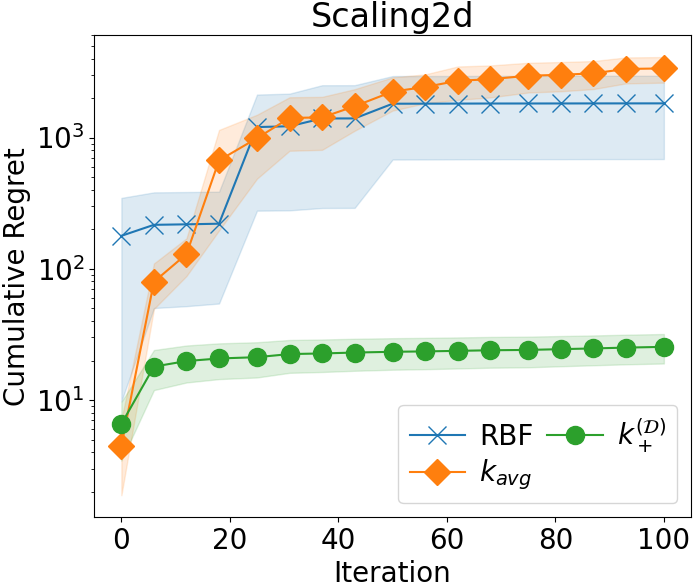} \quad
    \includegraphics[height=3.2cm]{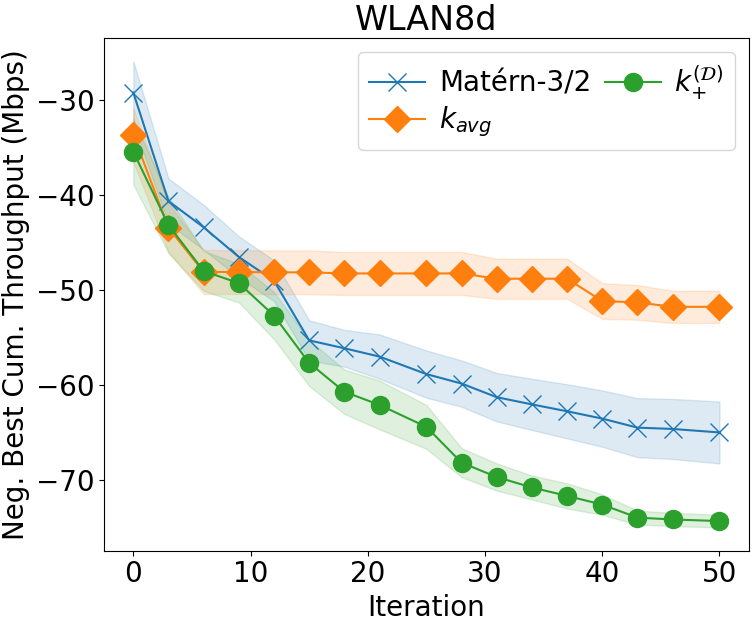}
    \caption{Cumulative regret and negated simple reward under GP-UCB with $\kb$ (blue crosses), $\kavg$ (orange diamonds), and $\kplus^{(\D)}$ (green circles). Shaded areas: standard error over 10 seeds.}
    \label{fig:xps}
\end{figure}

\textbf{Headline: $\kplus^{(\D)}$ wins on every task.}
Across all benchmarks (\Cref{tab:results}), $\kplus^{(\D)}$ achieves the best mean performance with up to 50\% of improvement. This answers \textbf{Q1} positively. Regarding \textbf{Q2}, we will see that as the group size increases, $\kplus^{(\D)}$ stays strong, while $\kavg$ degrades and can even underperform the non-invariant base kernel $\kb$.

\textbf{Setup in one glance.}
We run GP-UCB with each kernel $k\in\{\kb,\kavg,\kplus^{(\D)}\}$, using the same acquisition and optimization budgets. We report results averaged over 10 seeds. Synthetic objectives span $d\in\{2,\dots,6\}$ and symmetry sizes from $|\G|=8$ up to continuous groups ($|\G|=\infty$). The real-world task is an 8-dimensional AP-placement problem invariant to AP permutations (\Cref{sec:xps-real_world}). Hyperparameters and group actions are detailed in \Cref{app:benchmarks}.

\subsection{Synthetic Benchmarks}
\label{sec:xps-synthetic}

We consider synthetic functions $\blackbox$ (Ackley, Griewank, Rastrigin, etc.) that exhibit symmetries and are classically considered as challenging to optimize in the BO literature~\citep{qian2021bayesian, bardou2024too}. We cover dimensions $d=2$ to $d=6$ and group sizes $|\G|=8$ to $|\G|=\infty$. We evaluate performance using the cumulative regret $R_T = \sum_{i=1}^T \big(\blackbox(\x^*) - \blackbox(\x_t)\big)$ since the global maximizer $\x^* = \argmax_{\x \in \Sp} \blackbox(\x)$ is known. 

\textbf{\textbf{Finite groups: the gap widens as $|\G|$ grows.}}
With Matérn-5/2 base $\kb$ on Ackley2d ($|\G|{=}8$), $\kavg$ and $\kplus^{(\D)}$ are tied; both dominate $\kb$.
As $|\G|$ increases (Griewank6d, $|\G|{=}64$; Rastrigin5d, $|\G|{=}3{,}840$), $\kplus^{(\D)}$ increasingly outperforms $\kavg$ achieving cumulative regrets that are, on average, $40\%$ and $49\%$ lower respectively (\Cref{tab:results} and \Cref{fig:xps}, top panels).

\textbf{\textbf{Continuous groups: $\kavg$ can underperform even $\kb$.}}
For radial and scaling invariances (continuous groups; RBF base), $\kavg$ degrades relative to $\kb$, while $\kplus^{(\D)}$ remains strong (\Cref{fig:xps}, bottom left). 

\subsection{Wireless Network Design}
\label{sec:xps-real_world}

A wireless network (WN) consists of $m$ access points (APs) deployed over a given area to provide Internet connectivity to $p$ users. Since the quality of service (QoS) of each AP is degraded by interference from neighboring APs, determining optimal AP placement is a central challenge in WN design~\citep{wang2020thirty}. In this benchmark, we use a simulator that, given $p$ users and $m$ APs placed on a surface $\mathcal{A}$, evaluates the resulting QoS (see \Cref{app:benchmarks} for details). The optimization task is therefore to determine the positions of $m$ APs on a two-dimensional surface, yielding the $2m$-dimensional search space $\Sp = \mathcal{A}^{m}$. Because all APs are identical, the QoS function is naturally invariant under permutations of their positions. We use a Matérn-3/2 base $\kb$ to better capture threshold effects in the objective induced by AP-user associations~\citep{bardou2022inspire}. Performance is evaluated using the negated best reward $\min_{t \in [T]} -\blackbox(\x_t)$ attained during optimization (the regret cannot be computed because the max of $\blackbox$ is unknown), since the goal is to assess the quality of the best network configuration discovered by the optimizer, rather than the cumulative negative reward across all explored configurations.

\textbf{\textbf{$\kplus^{(\D)}$ finds better network configurations.}} %delivers the best QoS.}}
In the AP-placement task with $p=16$ users and $m=4$ APs ($d=8,|\G|=24$ permutations), $\kplus^{(\D)}$ consistently discovers higher-throughput configurations than both $\kavg$ and $\kb$ (\Cref{fig:xps}, bottom right; \Cref{fig:wlan} in \Cref{app:benchmarks} contains the resulting network configuration).

\subsection{Robustness to Group Size}

Both synthetic and real-world benchmarks suggest that $\kavg$ performs comparably to $\kplus^{(\D)}$ when the group size $|\G|$ is small, but its performance deteriorates as $|\G|$ grows, whereas $\kplus^{(\D)}$ remains stable. To investigate this effect more systematically, we conduct additional experiments on the $d$-dimensional Ackley and Rastrigin benchmarks, each invariant under the hyperoctahedral group $\G$ of size $|\G| = 2^d d!$. We compare the average regret of $\kavg$ and $\kplus^{(\D)}$ after 50 iterations of GP-UCB for $d=1,\dots,5$, and include $\kb$ as a baseline to control for the effect of increasing dimensionality.

\begin{figure}[t]
    \centering
    \begin{subfigure}{0.32\textwidth}
        \centering
        \includegraphics[height=3.2cm]{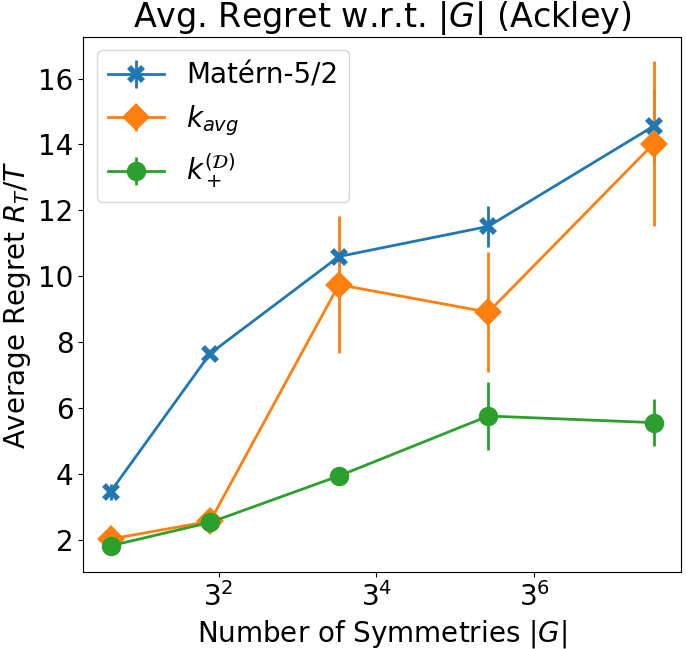}
    \end{subfigure}%
    \begin{subfigure}{0.32\textwidth}
        \centering
        \includegraphics[width=\linewidth]{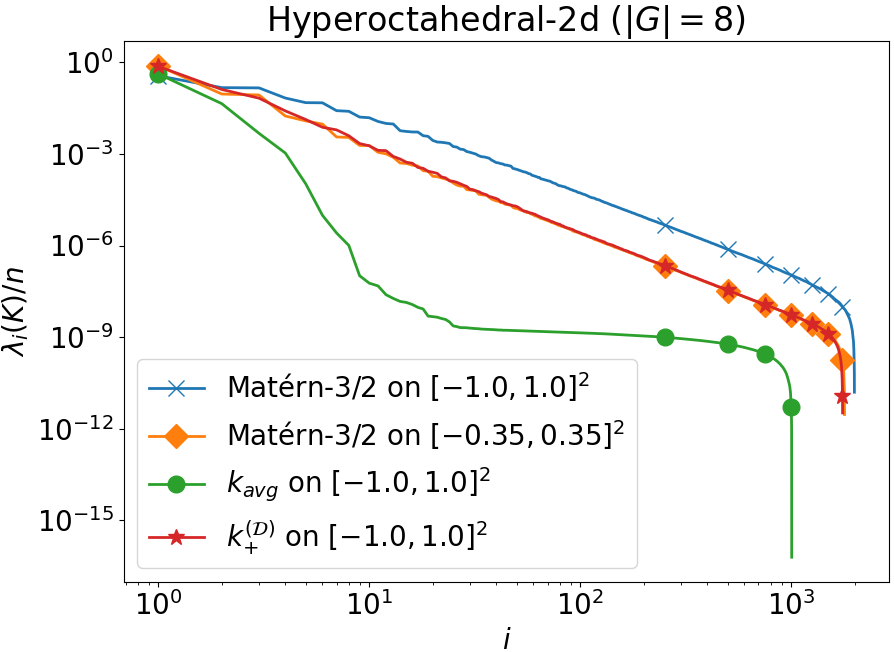}
    \end{subfigure}%
    \begin{subfigure}{0.32\textwidth}
        \centering
        \includegraphics[width=\linewidth]{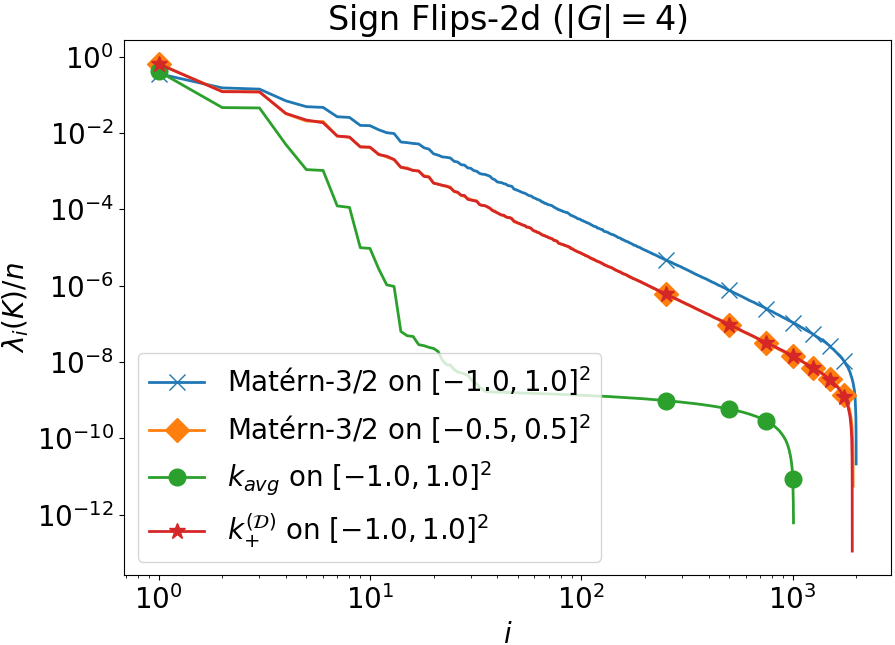}
    \end{subfigure}

    \begin{subfigure}{0.32\textwidth}
        \centering
        \includegraphics[height=3.2cm]{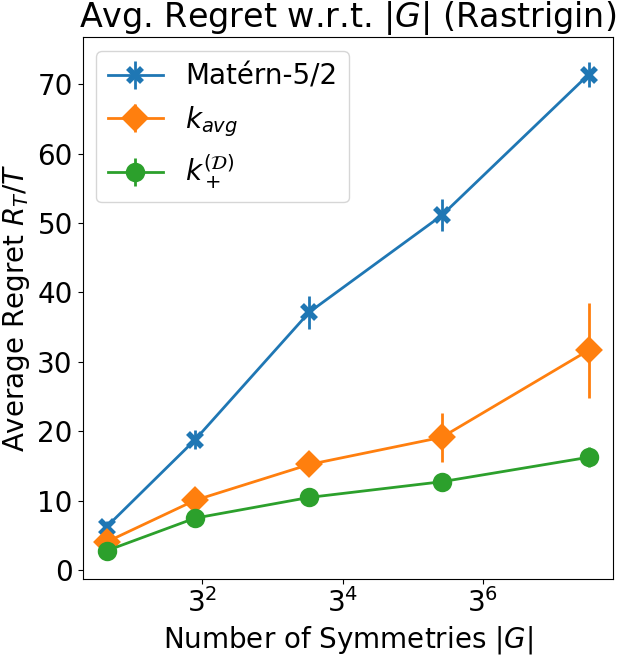}
    \end{subfigure}%
    \begin{subfigure}{0.32\textwidth}
        \centering
        \includegraphics[width=\linewidth]{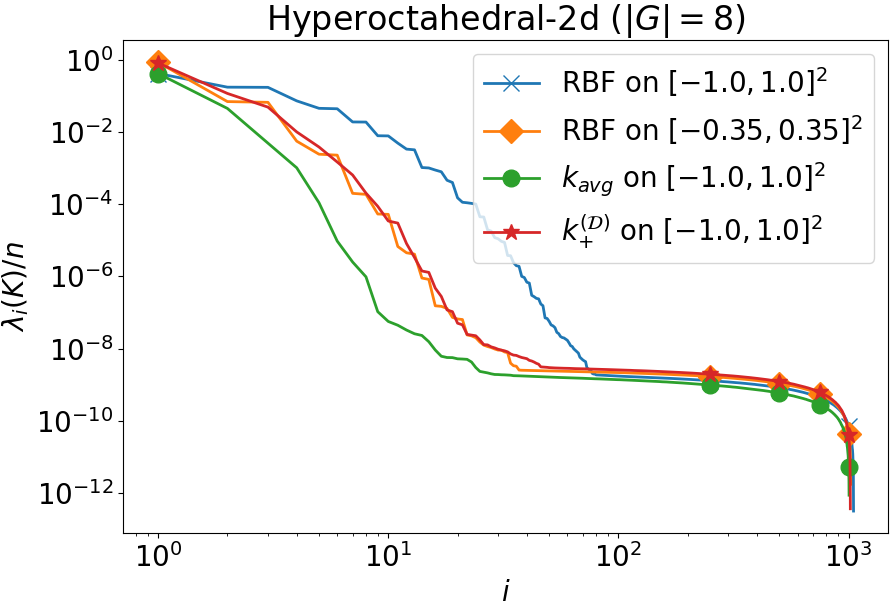}
    \end{subfigure}%
    \begin{subfigure}{0.32\textwidth}
        \centering
        \includegraphics[width=\linewidth]{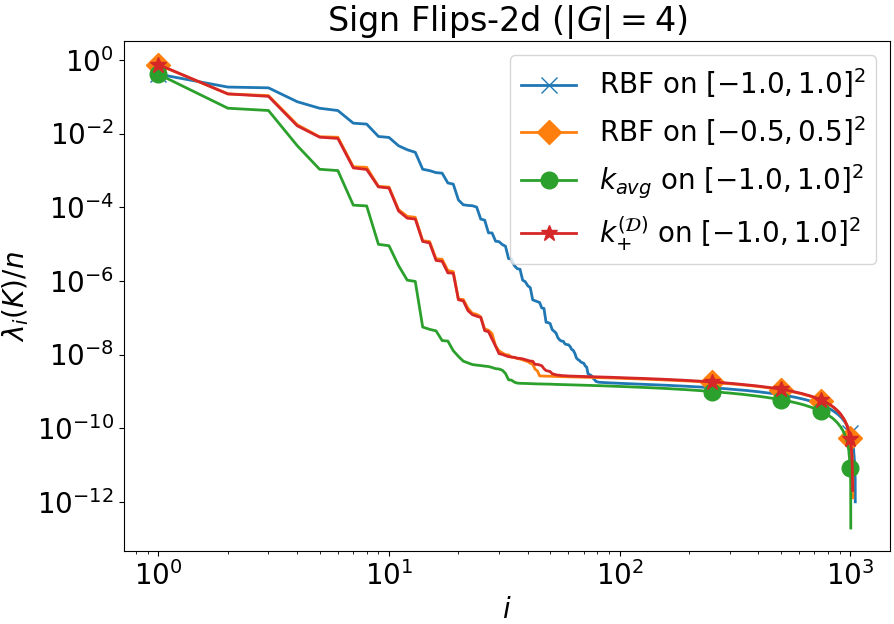}
    \end{subfigure}
    \caption{\textbf{Left column:} Average regret $R_T/T$ for $\kb$ (blue crosses), $\kavg$ (orange diamonds), and $\kplus^{(\D)}$ (green circles) on Ackley (top) and Rastrigin (bottom), averaged over 10 seeds with standard error bars.  
    \textbf{Middle and right columns:} Empirical eigendecays under different bases and groups (ordered, normalized eigenvalues of the Gram matrix).}
    \label{fig:merged}
\end{figure}

The results are shown in 
\Cref{fig:merged} (left column)
% \Cref{fig:perf_vs_symmetries}
. Both experiments reveal the same trend: while $\kavg$ consistently outperforms $\kb$, its performance also deteriorates as $|\G|$ increases. In contrast, $\kplus^{(\D)}$ remains largely unaffected by the growing number of symmetries, demonstrating a clear robustness to group size. In the next section, we discuss several explanations for these empirical observations.

\textbf{\textbf{Takeaway.}}
$\kplus^{(\D)}$ consistently matches or outperforms $\kavg$ and $\kb$, with the largest gains at large $|\G|$. The evidence suggests that (i) \emph{how} a kernel encodes orbit alignments matters as much as \emph{whether} it is invariant, and (ii) averaging across many alignments can dilute informative similarities. These themes reconnect with our discussion in \Cref{sec:eigendecay} and motivate analyses beyond eigendecay rates. 

\section{Spectral Analysis and Regret Bounds}
\label{sec:eigendecay}

So far, $\kplus^{(\D)}$ has shown consistently lower regret than $\kavg$, despite comparable
computational cost. A natural question is: \emph{can existing BO theory account for such a gap?}
Current regret bounds for GP surrogates proceed via the information gain, which is shaped by the decay of the operator spectrum of the kernel. In particular, faster spectral decay leads to tighter regret upper bounds in standard analyses
\citep{srinivas-information:2012, valko2013finite, scarlett2017lower, whitehouse2023sublinear}.
We now compare the eigendecay of $\kplus^{(\D)}$ and $\kavg$, and ask whether it can
explain the empirical gap.

\textbf{Empirical eigendecays: similar or \emph{faster} decay for $\kavg$.}
Across our benchmarks, the empirical spectra of $\kplus^{(\D)}$ and $\kavg$ exhibit very similar
log--log slopes (decay rates). In several settings, $\kavg$’s eigenvalues decay even
\emph{faster} than those of $\kplus$; see 
\Cref{fig:merged} (middle and right columns).
Under the usual theory, this would translate into similar, or potentially \emph{tighter}, upper
bounds for methods run with $\kavg$ compared to those with $\kplus^{(\D)}$. A more detailed discussion of the empirical spectra in 
\Cref{fig:merged} 
and further insights are in \Cref{app:eigendecay}. 

\textbf{Limitations of eigendecay as an explanation.}
Since $\kavg$ matches or exceeds $\kplus^{(\D)}$ in empirical decay rate, standard theory would predict similar or better regret upper bounds. Yet in practice we consistently observe lower regret for $\kplus^{(\D)}$ (\Cref{sec:xps}). This suggests that eigendecay alone does not capture the structural advantages of $\kplus^{(\D)}$. We outline possible explanations in the conclusion.

\section{Conclusion}
\label{sec:conclusion}

Our spectral analysis highlights a gap between theory and practice: although $\kavg$ often
exhibits \emph{faster} empirical eigendecay than $\kplus^{(\D)}$, the latter consistently
achieves lower regret. Standard eigendecay arguments thus fail to explain the observed advantage
of $\kplus^{(\D)}$.

We hypothesize two complementary explanations.  
First, \textbf{geometry vs.\ rates:} eigendecay quantifies how fast spectra shrink but ignores
\emph{which} eigenfunctions are emphasized. In practice, $\kavg$ often introduces
\emph{similarity reversals}, distorting the search geometry
(\Cref{fig:motivating_example}), whereas $\kplus^{(\D)}$ preserves high-contrast alignments between orbits, inherited from $\kmax$.  
Second, \textbf{approximation hardness:} BO theory typically assumes that the black-box
$\blackbox$ lies in the RKHS $\mathcal{H}_k$ of the chosen kernel $k$. Existing work on
\emph{misspecification} \citep{bogunovic2021misspecified} shows that the cumulative regret can be bounded from below by a linear term that involves the distance between $\blackbox$ and $\mathcal{H}_k$. Yet even when this distance is zero,
different kernels may yield very different approximation rates, affecting how quickly BO can optimize $\blackbox$. This distinction matters: in our experiments with the RBF kernel as $\kb$ (\Cref{sec:xps}), $\mathcal{H}_{\kb}$ is universal (property of the RBF kernel, see~\citet{micchelli2006universal}),
hence invariant functions $\blackbox$ always lie in $\mathcal{H}_{\kavg}$ (consider $(Pf)(\x)=\sum_{g\in\G} f(\act{g}{\x})/|\G|$ the projection onto $\mathcal{H}_{\kavg}$~\citep[Appendix A]{Brown24SampleEfficientBOInvariant} and observe that if $f_n\to \blackbox$ with $f_n\in \mathcal{H}_{\kb}$ then $Pf_n \to \blackbox$ with $Pf_n \in\mathcal{H}_{\kavg}$). There is no
misspecification in the sense of \citet{bogunovic2021misspecified} since $d(\blackbox,\mathcal{H}_{\kavg})=0$, yet $\kavg$ still performs
worse than $\kplus^{(\D)}$. This suggests that $\blackbox$ is simply \emph{harder to approximate}
in $\mathcal H_{\kavg}$ than in $\mathcal H_{\kmax}$. 
A plausible reason why
\citet{Brown24SampleEfficientBOInvariant} report strong performance for
$\kavg$ is that they focus on functions that are explicit linear combinations
of relatively few $\kavg(\x_t,\cdot)$ atoms (between $64$ and $512$, depending
on dimension; see their Appendix~B.1). In such settings, $\kavg$ looks very effective since its GP posterior mean can in principle
recover the function exactly once those $x_t$ are sampled. Typical BO objectives do not share this structure, which
may explain why in our experiments $\kavg$ sometimes underperforms even the
base kernel, while $\kplus^{(\D)}$ remains more reliable. Developing regret bounds that also measure \emph{approximation hardness}, capturing both the distance to $\mathcal{H}_k$ and approximation rates, seems a promising way to obtain guarantees that align more closely with empirical performance.  

Finally, while our focus has been empirical, we note that the intrinsic data-independent version of $\kplus^{(\D)}$, which we called $\kplus$ and which we mentioned at the end of \Cref{sec:max_ker-kplus} (introduced formally in \Cref{app:kplus}), provides a natural, data-independent analogue of the practical kernel $\kplus^{(\D)}$. We see $\kplus$ as a convenient object for future theoretical work, as it cleanly isolates the PSD projection of $\kmax$ from the additional data dependence introduced by Nyström. We believe that it makes $\kplus$ a convenient starting point for any future theoretical work, in the same spirit as gradient flow serving as an idealized analogue of gradient descent.

\subsubsection*{Acknowledgments}

This work was supported in part by the Swiss State Secretariat for Education, Research and Innovation (SERI) under contract number MB22.00027.

\newpage

\appendix

\section{Proofs for \Cref{sec:max_ker}} \label{app:kmax_proofs}

\subsection{Full statement and proof of \Cref{prop:kmax_is_adequate}}
We state \Cref{prop:kmax_is_adequate} formally and give a slightly more detailed proof.
\begin{proposition}[Max-kernel covariance for invariant GPs]
\label{prop:kmax_is_adequate_app}
Let $\Sp,\Sp_h\subset\R^d$ be measurable spaces and let a (finite or compact) group $\G$ act measurably on $\Sp$.
Let $h\sim\GP(0,\kb)$ be a GP on $\Sp_h$ with an isotropic base kernel $\kb:(\x,\x')\in\Sp\times\Sp\mapsto \kappa(\|\x-\x'\|_2)$ where $\kappa:\Rp\to\Rp$ is nonincreasing.
Assume there exists $\phi_\G:\Sp\to\Sp_h$ satisfying (i) \emph{invariance:} $\phi_\G(\x)=\phi_\G(\act{g}{\x})$ for all $g\in\G, \x\in\Sp$; and
(ii) \emph{minimal-distance representativity:} $\|\phi_\G(\x)-\phi_\G(\x')\|_2=\min_{g,g'\in\G}\|\act{g}{\x}-\act{g'}{\x'}\|_2$.
Define $f(\x)=h(\phi_\G(\x))$. Then $f\sim\GP(0,\kmax)$ and it is $\G$-invariant.
\end{proposition}

\begin{proof}
Since $g$ is a GP, $f$ is also a GP, and invariance follows from (i). Its covariance kernel is $\kmax$ since:
\begin{align}
        \Cov\left[f(\x), f(\x')\right] &= \Cov\left[h(\phi_\G(\x)), h(\phi_\G(\x'))\right] \nonumber\\
        &= \kb(\phi_\G(\x), \phi_\G(\x')) \nonumber\\
        &= \kappa(\min_{g,g' \in \G} ||\act{g}{\x} - \act{g'}{\x'}||_2) \label{eq:proof_isotropy_and_phi_prop}\\
        &= \max_{g,g' \in \G} \kappa(||g\x - g'\x'||_2) \label{eq:proof_kappa_nonincreasing}\\
        &= \kmax(\x, \x')
\end{align}
where we used (ii) in \Cref{eq:proof_isotropy_and_phi_prop}, and monotonicity of $\kappa$ in \Cref{eq:proof_kappa_nonincreasing}. Note that compactness of $G$ guarantees that the minimum in (ii) is indeed achieved, which makes \Cref{eq:proof_kappa_nonincreasing} true even when $\kappa$ is not necessarily continuous. 
\end{proof}

\subsection{Averaging vs Maximization with Different Base Kernels}
\label{app:kavg_vs_kmax}

We extend \Cref{prop:kavg_does_not_fit} to the case where 
$\kavg$ and $\kmax$ are built from \emph{different} base kernels.
The result shows that even in this more flexible setting, the coincidence of $\kavg$ and $\kmax$ can only occur in degenerate situations.

\begin{lemma}\label{prop:kavg_does_not_fit_diffkernels}
Let $\kb$ and $\kb'$ be two base kernels such that 
$\|\kb\|_\infty=\|\kb'\|_\infty=1$ and $\kb'(\x,\x)=1$ for all $\x$.  
Let $\kavg$ be the group-averaged kernel built from $\kb$ and $\kmax$ be the maximization kernel built from $\kb'$.  
It holds
\[
  \kavg = \kmax \quad \text{on the orbit } 
  \mathcal O(\x,\act{g}{\x}) := \{(\act{h}{\x},\act{h'}{\act{g}{\x}}), \; h,h'\in\G\}
\]
for every $\x\in\mathcal X$ and $g\in\G$, 
if and only if
\[
  \kb(\x,\act{g}{\x}) = \kmax(\x,\act{g}{\x}) = 1 \quad \text{for every $\x$ and $g\in\G$}.
\]
In particular, this forces $\kb$ to already exhibit a form of $\G$-invariance on pairs $(\x,\act{g}{\x})$.
\end{lemma}

\begin{proof}
($\Rightarrow$) 
Fix $\x$ and $g\in\G$.  Since by assumption $\kb'$ is bounded by $1$ and $\kb'(\x,\x)=1$:
\[
 1 \geq \kmax(\x,\act{g}{\x}) 
   = \max_{h,h'\in\G}\kb'(\act{h}{\x},\act{h'}{\act{g}{\x}}) \geq \kb'(\x,\x)=1
\]
so $\kmax(\x,\act{g}{\x}) = 1$.

Now consider $\kavg$. By definition,
\[
 \kavg(\x,\act{g}{\x})
   = \frac{1}{|\G|^2}\sum_{h,h'\in\G}
      \kb(\act{h}{\x},\act{h'}{\act{g}{\x}}).
\]
Each summand is bounded by $1$ and the average is equal to $1$ as $ \kavg(\x,\act{g}{\x}) = \kmax(\x,\act{g}{\x}) = 1$. Therefore each term is equal to $1$, which proves $\kb=\kmax=1$ on $\mathcal{O}(\x,\act{g}{\x})$. As this is true for every $\x,g\in\G$, this shows the result. The converse is immediate.
\end{proof}

This shows that even when allowing different base kernels for $\kavg$ and $\kmax$, equality between the two kernels requires $\kb$ to already be argumentwise $\G$-invariant on pairs $(\x,\act{g}{\x})$.  This fails for standard choices (e.g.\ RBF kernels with translation or rotation groups),  so averaging cannot replicate maximization in practice.

\section{Radial invariance: closed form for $\kavg$}
\label{app:radial_kavg}

We prove the formulas provided in \Cref{ex:radial}. 
Let $\G=\mathrm{SO}(2)$ act on $\R^2$ by in-plane rotations, and let $\kb$ be the RBF kernel with lengthscale $l$:
$\kb(\x,\x')=\exp\!\big(-\|\x-\x'\|_2^2/(2l^2)\big)$. Writing $\x=(r,\theta)$ and $\x'=(s,\varphi)$ in polar coordinates, we have
\[
\kavg(\x,\x') = \frac{1}{(2\pi)^2}\int_{0}^{2\pi}\!\!\int_{0}^{2\pi}
\exp\!\Big(-\tfrac{r^2+s^2-2rs\cos(\theta-\varphi+\alpha-\beta)}{2l^2}\Big)\, d\alpha\, d\beta.
\]
Integrating out the absolute angle and keeping only the relative angle $\psi=\theta-\varphi+\alpha-\beta$ yields
\[
\kavg(\x,\x') = \exp\!\left(-\tfrac{r^2+s^2}{2l^2}\right)\cdot \frac{1}{2\pi}\int_{0}^{2\pi}
\exp\!\left(\tfrac{rs}{l^2}\cos\psi\right)\,d\psi
= \exp\!\left(-\tfrac{r^2+s^2}{2l^2}\right) I_0\!\left(\tfrac{rs}{l^2}\right),
\]
where $I_0(z)=\tfrac{1}{2\pi}\int_0^{2\pi}e^{z\cos\psi}\,d\psi$ is the modified Bessel function of order~0.

\section{An intrinsic PSD projection $\,\kplus\,$ and its properties}
\label{app:kplus}

In the main text we defined a \emph{data-dependent} kernel $\kplus^{(\D)}$, corresponding to a PSD projection of $\kmax$ on a finite set of samples $\D$, extended by Nyström. This finite-sample construction $\kplus^{(\D)}$ is the star of the show in practice (as it is convenient to compute, and shows strong performance in practice). However, its data-dependence might make theoretical analysis quite involved. In this appendix, we show that $\kplus^{(\D)}$ is the finite-sample facet of a broader, intrinsic \emph{data-independent} PSD projection $\kplus$ of $\kmax$ which (i) preserves the
$\G$-invariance of $\kmax$, (ii) coincides with $\kmax$ whenever $\kmax$ is already
PSD. 
Since the PSD projection of $\kmax$ discussed here can also be applied to any other indefinite kernel $k$, we directly introduce it for an arbitrary kernel $k$. 

We begin as a warmup with the finite-domain “matrix” construction to build intuition, and then lift it to general domains via integral operators.

\subsection{Warmup: finite domains}
\label{app:kplus:finite}

We start on a finite domain $\Sp$ to build intuition. In that case, 
$\kplus$ is simply Frobenius-nearest PSD truncation of the Gram matrix on the \emph{full domain} $\Sp$, which is unique,
basis-independent, preserves $\G$-invariance, and coincides with $k$ when $k$ is already PSD. 

Let $\Sp=\{\x_1,\ldots,\x_N\}$ be finite, and let $\G$ act on $\Sp$. Consider any symmetric kernel
$k$ on $\Sp$ with Gram matrix $\K\in\R^{N\times N}$ (possibly indefinite) given by $\K_{ij}=k(\x_i,\x_j)$. We define $\kplus$ as the kernel corresponding to the Frobenius-nearest PSD projection of $\K$ \citep{HIGHAM1988103}. 

\begin{lemma}[Frobenius PSD projection and explicit form \citep{HIGHAM1988103}]
\label{lem:finite-psd-projection}
The optimization problem $\;\K_+ := \argmin_{\bm P\succeq 0}\|\bm P-\K\|_F\;$ has a unique solution and,
for any eigendecomposition $\K=Q\Lambda Q^\top$, it is given by 
\[
\K_+ \;=\; \bm Q\,\max(0,\bm \Lambda)\,\bm Q^\top ,
\]
where $\max(0,\cdot)$ acts entrywise on $\bm \Lambda$. In particular, the matrix $\K_+$ depends
only on $\K$ (not on the chosen eigenbasis), satisfies $\K_+\succeq 0$, and $\K_+=\K$ iff $\K\succeq 0$. 
\end{lemma}

\noindent
We \emph{define} $\kplus$, the (Frobenius) PSD projection of $k$, as:
\begin{equation}
\label{eq:finite-Kplus}
\kplus(x_i,x_j) \;:=\; (\K_+)_{ij}, \qquad i,j\in[N].
\end{equation}

\paragraph{Inheritance of $\G$-invariance.}

Each element $g\in\G$ induces a permutation of the elements of $\Sp$: let  $\pi_g$ be the permutations of the integers $j\in\{1,\dots,N\}$ defined by $\act{g}{\x_j} = \x_{\pi_g(j)}$.  
Denote by $\bm P_g$ the permutation matrix associated with $\pi_g$. For every vector $\bm v$, the matrix $\bm P_g$ acts as $(\bm P_g \bm v)_i = \bm v_{\pi_g^{-1}(i)}$ which is equivalent to the action on canonical vectors $\bm P_g \bm e_j = \bm e_{\pi_g(j)}$ or $(\bm P_g)_{ij}=1_{i=\pi_{g}(j)}$. 

Invariance in the first component guarantees $\kmax(\x_{\pi_g(i)}, \x_j) = \kmax(\act{g}\x_i, \x_j) = \kmax(\x_i, \x_j)$ for every $i,j\in\{1,\dots,N\}$, i.e., the rows of $\K=(k(\x_i,\x_j))_{i,j}$ are invariant under the permutation $\pi_g$, hence $\bm P_g \K = \K$. Thus, for any positive integer $m$, $\bm P_g \K^m = (\bm P_g \K)\K^{m-1} = \K^m$ so for any polynomial $p$ such that $p(0)=0$, $\bm P_g p(\K) = p(\K)$. Now consider a sequence $(p_n)_n$ of polynomials such that\footnote{We can impose $p_n(0)=0$ since $f(0)=0$. Indeed, take $p_n(\lambda) = q_n(\lambda) - q_n(0)$ where $q_n$ is a sequence given by Weierstrass' theorem, which converges to $f(\lambda)=\max(0,\lambda)$ on the spectrum of $\K$. We have $|p_n(\lambda) - f(\lambda)| \leq |q_n(\lambda) - f(\lambda)| + |q_n(0)|$ and because $f(0)=0$ we get $|q_n(0)| = |q_n(0) - f(0)|\to 0$.} $p_n(0)=0$ and $|p_n(\lambda) - \max(0,\lambda)|\underset{n\to\infty}{\to} 0$ for any $\lambda$ in the spectrum of $\K$. In the limit 
$\bm P_g \K_+ = \K_+$,
% $P_g \K_+ = P_g Q (\lim_{n\to\infty} p_n(\Lambda)) Q^\top = \lim_{n\to \infty} P_g p_n(\K) = \lim_{n\to\infty} p_n(\K) = \K_+$
hence $\kplus$ is invariant under the action of $\G$ on the first variable ($\kplus(\act{g}{\x},\x') = \kplus(\x,\x')$), and invariance along the second one follows by symmetry ($\K_+ \bm P_g^\top = \K_+$). This shows that $\kplus$ inherits from the $\G$-invariance of $k$ (equivalently, $\bm P_g \K = \K = \K \bm P_g^\top$ for all $g$). We collect this result in the next lemma.

\begin{lemma}[Invariance is preserved by the projection]
\label{lem:finite-invariance}
Consider $g\in\G$. If $P_g \K = \K$, then $P_g \K_+ = \K_+ = \K_+ P_g^\top$. Hence the
projected kernel $\kplus$ is $\G$-invariant on $\Sp\times\Sp$.
\end{lemma}

\paragraph{Relation to the practical Nyström kernel.}
If the set $\D=\{\x_1,\ldots,\x_n\}$ used to build $\kplus^{(\D)}$ (\Cref{eq:nystrom-kplus}) equals the whole domain $\D=\Sp$, then $\kplus^{(\D)}=\kplus$. Indeed, $\kplus^{(\D)}(\x_i,\x_j)=\K_{i:} \K^\dagger_+ \K_{:j} = (\K\K^\dagger_+\K)_{ij}=(\K_+)_{ij}$ on $\D\times \D$, and the latter is the definition of $\kplus$ on finite domains.

We now generalize the matrix considerations above using integral operators. The finite-domain construction is recovered as a special case. 

\subsection{General definition (via integral operators theory)}

We lift the finite-domain construction of the previous subsection to general domains by viewing $k$ as a
Hilbert–Schmidt operator and defining $\kplus$ as the positive part of $T_k$; this yields a PSD,
data-independent kernel that inherits any $\G$-invariance and equals $k$ whenever $k$ is PSD.

Let $(\Sp,\mathcal{T},\mu)$ be a probability space. 
For a measurable, symmetric kernel $k:\Sp\times\Sp\to\R$ with $k\in L^2(\mu\otimes\mu)$,
let the (compact, self-adjoint) Hilbert-Schmidt operator $T_k:L^2(\mu)\to L^2(\mu)$ be
\[
  (T_k f)(\x)\;=\;\int_{\Sp} k(\x,\x')\,f(\x')\,d\mu(\x').
\]
(Note that in the finite-domain case, $f$ is a vector indexed by the domain and if $\mu$ is the uniform measure then $T_k$ is simply multiplication by the Gram matrix $\K$ normalized by the domain size.) 
By the spectral theorem, there exist $(\lambda_i,\phi_i)_{i\ge1}$ with $\{\phi_i\}$ orthonormal in
$L^2(\mu)$ and $(\lambda_i)\in\ell^2$ (possibly of mixed signs) such that $T_k=\sum_{i\ge1}\lambda_i\,\phi_i\otimes\phi_i$ in $L^2(\mu)$ where for every $u,v\in L^2(\mu)$, $u\otimes v$ is the rank-one operator $L^2(\mu) \to L^2(\mu)$ such that $(u\otimes v)f := \langle f,v\rangle\,u$ for every $f\in L^2(\mu)$.

\paragraph{Generic definition of $\kplus$ via operator theory.}
Define the positive part of $T_k=\sum_i\lambda_i\,\phi_i\otimes\phi_i$ by $T_k^+:=\sum_i (\lambda_i)_+\,\phi_i\otimes\phi_i$, where $(t)_+=\max\{t,0\}$. Since $\sum_i ((\lambda_i)_+)^2 \leq \sum_i \lambda_i^2 < \infty$, the series
\begin{equation}
  \label{eq:kplus-spectral}
  \kplus(\x,\x') \;:=\; \sum_{i\ge1} (\lambda_i)_+\,\phi_i(\x)\,\phi_i(\x') 
  \quad\text{($\mu\otimes\mu$-a.e.)}.
\end{equation}
converges in $L^2(\mu\otimes\mu)$ and defines a kernel $\mu\otimes\mu$-almost everywhere.
By construction\footnote{Indeed, by definition $(T_{\kplus}f)(\x)
= \int_\Sp \Big(\sum_{i\ge1}(\lambda_i)_+ \phi_i(\x)\phi_i(\x')\Big) f(\x')\,d\mu(\x')=\sum_{i\ge1} (\lambda_i)_+ \,\langle f,\phi_i\rangle\, \phi_i(\x) = \Bigl(\Big(\sum_{i\ge1} (\lambda_i)_+\,\phi_i\otimes\phi_i\Big) f\Bigr) (\x)= \big(T_k^+ f\big)(\x)$.} $T_{\kplus}=T_k^+$, hence $\kplus$ is PSD as a kernel a.e., and PSD in the operator sense: $\inner{f}{T_{\kplus}f}\geq 0$ for all $f\in L^2(\mu)$.
In particular, if $k$ was already PSD (all $\lambda_i\ge0$), then $\kplus=k$ (up to null sets). 
It also inherits $\G$-invariance of $k$ if $k$ is indeed invariant (the proof mimics the finite-domain case, we give the full details for completeness in \Cref{subsec:proof-invariance-kplus}).

\subsection{From the finite-sample projection to the intrinsic limit: what converges to what?}
\label{app:kplus-approx}

We relate the practical, data-dependent Nyström kernel
$k_+^{(\D)}$ (\Cref{eq:nystrom-kplus}) to the intrinsic $\kplus$: under iid sampling, the empirical spectra of $k_+^{(\D)} / |\D|$
converge to that of $T_{\kplus}$, with rates under mild moment assumptions. This shows that eigendecay-based regret analysis 

\textbf{Notations.} Let $X_1,X_2,\dots\sim\mu$ i.i.d.\ and $\D_n=\{X_1,\dots,X_n\}$. We write 
$\K_n:=k(\D_n,\D_n)$, $\K_n^+:=\arg\min_{\bm P\succeq 0}\|\bm P-\K_n\|_F$,
$\tilde \K_n:=\K_n/n$, and recall that the practical (data-dependent) kernel defined in \Cref{eq:nystrom-kplus} is
\[
k_+^{(\D_n)}(\x, \x') \;=\; k(\x,\D_n)\,(\K_n^+)^{\dagger}\,k(\D_n,\x').
\]
We denote by $\lambda(T)$ the (ordered, nonincreasing, each counted with its multiplicity) sequence of eigenvalues of a compact
self-adjoint operator $T$, and by
$\delta_2\!\big(\lambda(T),\lambda(S)\big):=\big(\sum_i|\lambda_i(T)-\lambda_i(S)|^2\big)^{1/2}$
the spectral $\ell_2$ distance. For symmetric matrices $\bm M$, $\lambda(\bm M)$ denotes the nonincreasing sequence of eigenvalues of $\bm M$ (with multiplicity) padded with an infinite number of zeros. For a bounded operator $A$, $\|A\|_{\mathrm{HS}}$ and
$\|A\|_{\mathrm{op}}$ denote the Hilbert-Schmidt and operator norms, respectively. We include in \Cref{app:conv-reminders} a reminder on the different notions of norms and convergence, and we now recall the essentials.

\paragraph{Relations between convergence notions.}
For compact self-adjoint operators:
(i) $$\max\left(\delta_2\!\big(\lambda(T_n),\lambda(T)\big),\|T_n-T\|_{\mathrm{op}}\right)\leq \|T_n-T\|_{\mathrm{HS}}$$
\citep{REED1972182,BhatiaElsner1994};
(ii) converse inequalities do not hold in infinite dimension (see \Cref{app:conv-reminders} for examples).
Thus, HS convergence is the strongest notion of convergence we manipulate here.

\medskip
We now present convergence guarantees of the
data-dependent construction $\kplus^{(\D_n)}/n$ to the intrinsic $\kplus$ under progressively stronger assumptions. With minimal assumptions we obtain almost-sure spectral consistency in the $\delta_2$ metric; with stronger assumptions we obtain quantitative rates in HS norm (hence also spectral $\ell_2$ in probability). 

\medskip
\noindent\textbf{(a) Weak a.s.\ spectral consistency of positive parts (minimal assumptions).}
\begin{proposition}\label{thm:weak-spectral-consistency}
Assume the symmetric (not necessarily PSD) kernel $k$ is in $L^2(\mu\otimes\mu)$ so that $T_k$ is Hilbert-Schmidt. Let $\widehat S_n:L^2(\mu_n)\to L^2(\mu_n)$ be the integral operator with kernel $k_+^{(\D_n)}(\x,\x')/n$ defined by:
\begin{equation}
\label{def:hat-S-n}
(\widehat S_n f)(\x) = \frac{1}{n} \sum_{j=1}^n \kplus^{(\D_n)}(\x, X_j)f(X_j).
\end{equation}

Assume the $X_i$ are pairwise distinct almost surely. Then, almost surely,
\[
\delta_2\!\Big(\lambda\!\big(\widehat S_n\big),\;\lambda\!\big(T_{\kplus}\big)\Big)\;\underset{n\to \infty}{\longrightarrow}\;0.
\]
\end{proposition}

\noindent\emph{Proof.} 
Let $\K_n$ be the empirical operator on $\R^n$ with matrix $\frac1n(k(X_i,X_j))_{i,j}$ and let $\lambda(\K_n)$ be its ordered spectrum (nonincreasing, with multiplicity) padded with an infinite number of zeros. Theorem 3.1 of \citet{koltchinskii2000random} shows that $\delta_2(\lambda(\K_n), \lambda(T_k)) \to 0$ as $n\to\infty$. 

Let $\K_n^+$ be the positive part of $\K_n$ (i.e., its Frobenius PSD projection). Since $\lambda\mapsto \max(0,\lambda)$ is $1$-Lipschitz, we have for any operators $T,S$:
\[
\delta_2(\lambda(T_+), \lambda(S_+)) = \sum_i |\max(0,\lambda_i(T)) - \max(0,\lambda_i(S))| \leq \sum_i |\lambda_i(T) - \lambda_i(S)| = \delta_2(\lambda(T), \lambda(S)).
\]
We deduce that $\delta_2(\lambda(\K_n^+), \lambda(T_{\kplus})) \to 0$ as $n\to\infty$.

It remains to observe that the spectrum of $\K_n^+$ as an operator on $\R^n$ is the same as $\widehat S_n:L^2(\mu_n)\to L^2(\mu_n)$. This identification is standard (e.g., see above Equation 1.2 in \citet{koltchinskii2000random}). For completeness, we include the formal arguments of \citet{koltchinskii2000random} in \Cref{lem:nystrom_empirical_identity}, which shows that we can identify the spectrum of $k_+^{(\D_n)}(\D_n,\D_n)/n$ with the one of $\K^+_n$ a.s. if the iid $X_i\sim \mu$ are pairwise distinct a.s, which is true as soon as $\mu$ is non-atomic; otherwise one can index the \emph{distinct} atoms and work in $\R^m$ with $m=\#\mathrm{supp}(\mu_n)$, obtaining the same spectral identity on that subspace. \qed

\medskip
\noindent\textbf{(b) Expected HS convergence with $\mathcal{O}(n^{-1/2})$ rate (stronger assumption).}
Define the empirical integral operator $(T_n f)(\x):=\frac1n\sum_{i=1}^n k(\x,X_i)f(X_i)$ and
$D_n:=T_n-T_k$. Let $(\lambda_i,\phi_i)_{i\ge1}$ be an eigensystem of $T_k$ in $L^2(\mu)$.
Assume the following fourth-order summability condition holds:
\begin{equation}
\label{eq:fourth-order-sum}
C\;:=\;\sum_{i,j\ge1}\lambda_i^2\int_\Sp \phi_i(\x)^2\,\phi_j(\x)^2\,d\mu(\x)\;<\;\infty.
\end{equation}

\begin{proposition}[Expected HS rate]\label{prop:expected-HS}
Under $k\in L^2(\mu\otimes\mu)$ and \eqref{eq:fourth-order-sum},
\[
\mathbb{E}\big[\|D_n\|_{\mathrm{HS}}^2\big]\;\le\;\frac{C}{n},
\qquad
\mathbb{E}\big[\|D_n\|_{\mathrm{HS}}\big]\;\le\;\sqrt{\tfrac{C}{n}}.
\]
Consequently, $\|D_n\|_{\mathrm{HS}}=\mathcal{O}_{\mathbb{P}}(n^{-1/2})$ and therefore using the same notations as in \Cref{thm:weak-spectral-consistency}
\[
\delta_2\!\big(\lambda(\K_n^+),\lambda(T_k^+)\big)
=\mathcal{O}_{\mathbb{P}}(n^{-1/2}),\qquad
\delta_2\!\Big(\lambda\!\big(\widehat S_n\big),\lambda(T_{\kplus})\Big)
=\mathcal{O}_{\mathbb{P}}(n^{-1/2}).
\]
\end{proposition}

\noindent\emph{Proof.} Fix any $f\in L^2(\mu)$. By Fubini-Tonelli for non-negative functions, we have:
\[
\mathbb{E}\big[\|D_n f\|_{L^2(\mu)}^2\big]
=\int_\Sp \E\Big[\big((D_n f)(\x)\big)^2\Big]\,d\mu(\x).
\]
By definition
\[
(D_n f)(\x) = \frac{1}{n}\sum_{i=1}^n k(\x,X_i)f(X_i) - \int_\Sp k(\x,\x')f(\x')\,d\mu(\x')
\]
where the randomness comes from the i.i.d.\ $X_i\sim\mu$. Hence $\mathbb{E}\big[(D_n f)(\x)\big]=0$ and for any fixed $\x$
\[
\E\Big[\big((D_n f)(\x)\big)^2\Big] = \Var\big((D_n f)(\x)\big)
=\frac{1}{n}\,\Var\big(k(\x,X)f(X)\big)
\le \frac{1}{n}\,\int_\Sp k(\x,\x')^2 f(\x')^2\,d\mu(\x').
\]
The Hilbert-Schmidt
spectral theorem gives the expansion
$k(\x,\x')=\sum_i \lambda_i \phi_i(\x)\phi_i(\x')$ in $L^2(\mu\otimes\mu)$, with $(\lambda_i)_i\in\ell^2$ and $(\phi_i)_i$ an orthonormal set of $L^2(\mu)$ (see Equation 3.2 in \citet{koltchinskii2000random}, Corollary 5.4 in \citet{Conway2007}). Thus
\begin{align*}
\int_\Sp \E\Big[\big((D_n f)(\x)\big)^2\Big]\,d\mu(\x) 
& \leq \frac{1}{n} \int_\Sp k(\x,\x')^2 f(\x')^2\,d\mu(\x')d\mu(\x) \\
& = \sum_{i,j} \lambda_i \lambda_j \int_\Sp \phi_i(\x') \phi_j(\x') f(\x')^2 \underbrace{\inner{\phi_i}{\phi_j}}_{=1_{i=j}} d\mu(\x') \\
& = \sum_i \lambda_i^2 \int_\Sp \phi_i(\x')^2 f(\x')^2 d\mu(\x').
\end{align*}
Taking $f=\phi_j$ for a fixed $j$ yields
\[
\mathbb{E}\big[\|D_n \phi_j\|_{L^2(\mu)}^2\big] \leq \frac{1}{n} \sum_i \lambda_i^2 \int_\Sp \phi_i(\x')^2 \phi_j(\x')^2 d\mu(\x').
\]
Since $\|D_nf\|_{\textrm{HS}}^2 = \sum_{j} \|D_n \phi_j\|_{L^2(\mu)}^2$, we get the main claim:
\[
\mathbb{E}\big[\|D_n\|_{\mathrm{HS}}^2\big]\;\le\;\frac{C}{n}.
\]
Jensen gives the bound for $\mathbb{E}\|D_n\|_{\mathrm{HS}}$. Finally,
$\delta_2(\lambda(\K_n),\lambda(T_k))\le \|D_n\|_{\mathrm{HS}}$ (Hoffman-Wielandt inequality in infinite dimension \citep{BhatiaElsner1994}),
and $\lambda\mapsto \max(0,\lambda)$ is $1$-Lipschitz on $\R$, hence the spectral bound probability claim using Markov's inequality, and \Cref{lem:nystrom_empirical_identity} transfers this claims to $\widehat S_n$. \qed

\begin{remark}[On assumption \eqref{eq:fourth-order-sum}]
Condition \eqref{eq:fourth-order-sum} is a fourth-order integrability requirement
that controls eigenfunction overlaps. It is standard in random Nyström analyses
(see, e.g., Equations (4.3) and (4.11) of \citet{koltchinskii2000random}) and stronger than $k\in L^2$,
but it yields a dimension-free $\mathcal{O}(n^{-1/2})$ rate in HS norm.
\end{remark}

\medskip
\noindent\textbf{(c) High-probability HS rates (heavier but more precise).}
Under slightly stronger $L^4$-type conditions on eigenfunctions, the section 4 in \citet{koltchinskii2000random} gives more more precise statements on the rates in \Cref{prop:expected-HS}, and we directly refer the reader to it. 

\paragraph{Application to $\kmax$ and to the BO kernels in the paper.}
When $k=\kmax$ is bounded on a compact domain $\Sp$ (as in all our experiments),
$k\in L^2(\mu\otimes\mu)$ for any probability measure $\mu$ on $\Sp$, so $T_{\kmax}$
is Hilbert-Schmidt and \Cref{thm:weak-spectral-consistency} applies. In particular, the integral operator associated with $\kplus^{(\D_n)} / n$, called $\widehat S_n$ (\Cref{def:hat-S-n}) satisfies
\[
\delta_2\!\Big(\lambda\!\big(\widehat S_n\big),\;\lambda\!\big(T_{\kplus}\big)\Big)
\;\xrightarrow[n\to\infty]{\text{a.s.}}\;0.
\]
This clarifies the two objects introduced in the main text: the \emph{intrinsic} $\kplus$
is the unique data-independent target, while the \emph{practical} kernel $\kplus^{(\D_n)}$
(finite PSD projection $+$ Nyström) is an on-path approximation whose spectrum converges (once normalized by $n$) to that of $\kplus$ under i.i.d.\ sampling. 

The following subsections are only optional complementary materials added to help building intuitions on the convergence results stated above.

\subsection{Reminders on the different type of convergences for bounded linear operators}
\label{app:conv-reminders}

This subsection recalls standard notions of operator convergence, 
included only as background to help build intuition for the convergence results above.

\paragraph{Definitions (operator norm, HS norm, spectral distance).}
Let $\mathcal H$ be a separable Hilbert space with orthonormal basis $\{e_i\}_{i\ge1}$.
For a bounded linear operator $T:\mathcal H\!\to\!\mathcal H$,
\[
\|T\|_{\mathrm{op}}:=\sup_{\|f\|_{\mathcal H}=1}\|Tf\|_{\mathcal H},\qquad
\|T\|_{\mathrm{HS}}:=\Big(\sum_{i\ge1}\|Te_i\|_{\mathcal H}^2\Big)^{1/2}.
\]
The HS norm is basis-independent. When $T$ is an \emph{integral} operator with kernel $k\in L^2(\mu\otimes\mu)$ on $L^2(\mu)$ \citep{REED1972182}
\[
\|T\|_{\mathrm{HS}}^2=\iint_{\Sp\times\Sp} |k(x,y)|^2\,d\mu(x)\,d\mu(y).
\]
For finite matrices, $\|A\|_{\mathrm{HS}}=\|A\|_F$ (Frobenius).
We say $T_n\!\to T$ in HS norm if $\|T_n-T\|_{\mathrm{HS}}\!\to 0$, and we say $T_n\!\to T$ spectrally if
$\delta_2\big(\lambda(T_n),\lambda(T)\big)\!\to 0$, where we recall that $\lambda(T)$ is the \emph{ordered}
eigenvalues of a compact self-adjoint operator $T$, and where the spectral $\ell_2$-distance is 
$\delta_2(\lambda(T),\lambda(S)) := \big(\sum_i|\lambda_i(T)-\lambda_i(S)|^2\big)^{1/2}$.

\paragraph{Which convergences matter, and how they relate (reminders on well-known facts).}
We compare three notions:
(i) \emph{operator norm} convergence $\|T_n-T\|_{\mathrm{op}}\!\to 0$;
(ii) \emph{Hilbert-Schmidt (HS)} convergence $\|T_n-T\|_{\mathrm{HS}}\!\to 0$;
(iii) \emph{spectral} convergence in $\delta_2$, i.e.,
$\delta_2\big(\lambda(T_n),\lambda(T)\big):=\big(\sum_i|\lambda_i(T_n)-\lambda_i(T)|^2\big)^{1/2}\to 0$,
where $\lambda(\cdot)$ denotes the ordered eigenvalues of a compact self-adjoint operator. We recall the following well-known facts, useful to grasp the convergence results we state next.

\medskip
\noindent\textbf{(1) HS $\Longrightarrow$ spectral $\delta_2$.}
For compact self-adjoint operators the (infinite-dimensional) Hoffman-Wielandt inequality yields \citep{BhatiaElsner1994}
\[
\delta_2\big(\lambda(T_n),\lambda(T)\big)\;\le\;\|T_n-T\|_{\mathrm{HS}}.
\]

\medskip
\noindent\textbf{(2) HS $\Longrightarrow$ operator norm.}
For every Hilbert-Schmidt operator $S$, $\|S\|_{\mathrm{op}}\le \|S\|_{\mathrm{HS}}$. 
Indeed for unit vectors $x,y\in H$, using $x=\sum_{i}\langle x,e_i\rangle e_i$, we have
\(
\langle Sx,y\rangle
=\sum_{i\in I} \langle x,e_i\rangle\,\langle Se_i,y\rangle.
\)
By Cauchy-Schwarz:
\[
|\langle Sx,y\rangle|
\le \Big(\sum_{i\in I}|\langle x,e_i\rangle|^2\Big)^{1/2}
     \Big(\sum_{i\in I}|\langle Se_i,y\rangle|^2\Big)^{1/2}.
\]
The first factor equals $\|x\|=1$, and for the second we use $|\langle Se_i,y\rangle|\le \|Se_i\|\,\|y\|=\|Se_i\|$ to get
\[
\sum_{i\in I}|\langle Se_i,y\rangle|^2 \le \sum_{i\in I}\|Se_i\|^2=\|S\|_{\mathrm{HS}}^2.
\]
Hence $|\langle Sx,y\rangle|\le \|S\|_{\mathrm{HS}}$. Taking the supremum over all unit $y$ gives
\[
\|Sx\|=\sup_{\|y\|=1}|\langle Sx,y\rangle| \le \|S\|_{\mathrm{HS}},
\]
and then taking the supremum over all unit $x$ yields
\[
\|S\|_{\mathrm{op}}=\sup_{\|x\|=1}\|Sx\| \le \|S\|_{\mathrm{HS}}.
\]
% Hence $\|T_n-T\|_{\mathrm{HS}}\!\to 0$ implies $\|T_n-T\|_{\mathrm{op}}\!\to 0$.

\medskip
\noindent\textbf{(3) Spectral $\delta_2$ does \emph{not} imply HS nor operator norm.}
Even if eigenvalues match in $\ell_2$, the operators may be far in norm because eigenvectors can rotate.
Let $T=\mathrm{diag}(1,1/2,1/3,\ldots)$ in the canonical basis $(e_i)_{i\ge1}$, and let $U_n$ swap $e_1$ and $e_n$.
Set $T_n:=U_n T U_n^\ast$. Then $\lambda(T_n)=\lambda(T)$ for all $n$ (same ordered spectrum), so
$\delta_2(\lambda(T_n),\lambda(T))=0$.
Yet $\|(T_n-T)e_1\|=\|(U_nTU_n^\ast - T)e_1\|=\|(1/n-1)e_1\|=1-1/n$, hence
$\|T_n-T\|_{\mathrm{op}}\ge 1-1/n\to 1$ and, a fortiori, $\|T_n-T\|_{\mathrm{HS}}\not\to 0$.

\medskip
\noindent\textbf{(4) Operator norm does \emph{not} imply spectral $\delta_2$.}
Let $T=0$ and $T_n$ be diagonal with the first $m_n$ entries equal to $\varepsilon_n$ and the rest $0$.
Choose $\varepsilon_n:=n^{-1/2}$ and $m_n:=n$. Then
$\|T_n\|_{\mathrm{op}}=\varepsilon_n\to 0$ but
$\delta_2\big(\lambda(T_n),\lambda(T)\big)
=\big(\sum_{i=1}^{m_n}\varepsilon_n^2\big)^{1/2}
=\sqrt{n\cdot (1/n)}=1$.

\medskip
\noindent\textbf{(5) Two useful corollaries.}
(a) Spectral $\delta_2$-convergence implies convergence of the \emph{largest} eigenvalue, since
$\sup_i|\lambda_i(T_n)-\lambda_i(T)|\le \delta_2(\lambda(T_n),\lambda(T))$.
(b) Operator-norm convergence forces uniform eigenvalue deviations to vanish by Weyl’s inequality:
$\sup_i|\lambda_i(T_n)-\lambda_i(T)|\le \|T_n-T\|_{\mathrm{op}}$,
but it does \emph{not} control the $\ell_2$-sum of all deviations.

\medskip
\noindent\emph{Takeaway.} HS is the strongest notion here: it simultaneously implies spectral $\delta_2$-convergence
(and thus convergence of eigenvalue-based quantities) and operator-norm convergence. The converses fail in infinite
dimension because eigenvectors can drift and an infinite number of tiny eigenvalue errors can accumulate.

\subsection{Identification of the spectrum of an empirical operator in $L^2(\mu_n)$ and its matrix counterpart}

Here we show how the spectrum of the empirical operator can be
identified with that of its matrix form. This is complementary material meant to clarify how
operator-level and matrix-level viewpoints connect (which is useful, e.g., in the proof of \Cref{thm:weak-spectral-consistency}).

\begin{lemma}[Empirical Nyström spectral identity]
\label{lem:nystrom_empirical_identity}
Let $\K_n:=\frac1n\big(k(\x_i,\x_j)\big)_{i,j=1}^n$ and let $\K_n^+$ be its spectral positive part
(the Frobenius-nearest PSD projection). Define the empirical measure
$\mu_n:=\frac1n\sum_{i=1}^n\delta_{\x_i}$ and the Nyström kernel
\[
k_+^{(\D_n)}(\x,\x') \;=\; k(\x,\D_n)\,(\K_n^+)^\dagger\,k(\D_n,\x').
\]
Let $\widehat S_n:L^2(\mu_n)\to L^2(\mu_n)$ be the integral operator with kernel $k_+^{(\D_n)}(\x,\x')/n$, i.e.
\[
(\widehat S_n f)(\x) \;=\; \frac{1}{n} \sum_{j=1}^n k_+^{(\D_n)}(\x, \x_j)\,f(\x_j).
\]
The map $E:L^2(\mu_n)\to\R^n$, $Ef:=\frac{1}{\sqrt{n}}\big(f(\x_1),\ldots,f(\x_n)\big)^\top$, is an isometry:
$\|Ef\|_{\R^n}=\|f\|_{L^2(\mu_n)}$, and we have the intertwining identity
\[
E\,\widehat S_n \;=\; \K_n^+\, E.
\]
If, in addition, the sample points $\x_1,\ldots,\x_n$ are pairwise distinct, then $E$ is an
isometric isomorphism (hence invertible) and
\[
\lambda\!\big(\widehat S_n\big)\;=\;\lambda\!\big(\K_n^+\big)\;=\;\lambda\!\big(k_+^{(\D_n)}(\D_n,\D_n)/n\big).
\]
\end{lemma}

\begin{proof}
First note the on-sample identity $k_+^{(\D_n)}(\x_i,\x_j)=(\K^+)_{ij}$ for the unscaled
$\K=(k(\x_i,\x_j))_{i,j}$, which follows from $\K(\K^+)^\dagger \K=\K^+$. Hence
$k_+^{(\D_n)}(\D_n,\D_n)=\K^+$ and therefore $k_+^{(\D_n)}(\D_n,\D_n)/n=\K_n^+$.

For $f\in L^2(\mu_n)$ and each $i\in\{1,\dots,n\}$,
\[
\sqrt{n}\,\big(E\widehat S_n f\big)_i
= (\widehat S_n f)(\x_i)
= \frac{1}{n}\sum_{j=1}^n k_+^{(\D_n)}(\x_i,\x_j)\,f(\x_j)
= \sum_{j=1}^n (\K_n^+)_{ij}\,f(X_j)
= \sqrt{n}\,\big(\K_n^+ Ef\big)_i,
\]
which proves $E\,\widehat S_n = \K_n^+ E$. Since $E$ is an isometry by definition of the
$L^2(\mu_n)$ inner product, if the $X_i$ are pairwise distinct then $E$ is bijective and
conjugates $\widehat S_n$ with $\K_n^+$, so the spectra (with multiplicities) coincide.
\end{proof}

\subsection{Proof of $\G$-invariance of $\kplus$ for general domains}
\label{subsec:proof-invariance-kplus}

We conclude this appendix with the formal proof that $\kplus$ defined in~\eqref{eq:kplus-spectral} inherits from any group-invariance of $k$. This proof is not needed for the main results but is included for
completeness. It makes explicit why $\kplus$ preserves any $\G$-invariance of $k$. The proof follows the one for finite domains but is heavier in notations because it is now stated using integral operators to generalize the matrix manipulations of finite domains. For finite domains, denoting by $\K$ the Gram matrix of $k$ over the whole domain and $\bm P_g$ the permutation matrix induced by the action of $g\in\G$ on the domain, invariance of $k$ is equivalent to $\bm P_g \K= \K \bm P_g^\top = \K$. Thus any polynomial $p(\K)$ of $\K$ such that $p(0)=0$ inherits from this invariance since we still have $\bm P_g p(\K) = p(\K) \bm P_g^\top = p(\K)$. And at the limit, we get invariance of $\K_+$. Here, we mimic this proof, and we start by introducing the equivalent integral operator form of the characterization $\bm P_g \K= \K \bm P_g^\top = \K$ for general domains.

\begin{lemma}[Kernel invariance $\Longleftrightarrow$ operator commutation]
\label{lem:inv_commute}
Let $(\Sp,\mathcal{T},\mu)$ be a probability space and let $\G$ act measurably on $\Sp$.
Assume $\mu$ is $\G$-invariant.  
Let $U_g:L^2(\mu)\to L^2(\mu)$ be the unitary representation
\(
  (U_g f)(\x) := f(\act{g^{-1}}{\x}).
\)
Let $k\in L^2(\mu\otimes\mu)$ be a symmetric kernel with integral operator
\(
  (T_k f)(\x)=\int_\Sp k(\x,\x')f(\x')\,d\mu(\x').
\)
Then the following are equivalent:
\begin{enumerate}
\item[\rm (i)] $k$ is argumentwise $\G$-invariant: $k(\act{g}{\x},\x')=k(\x,\act{g}{\x'})=k(\x,\x')$ for $\mu\otimes\mu$-a.e.\ $(\x,\x')$ and all $g\in\G$.
\item[\rm (ii)] $T_k$ satisfies $U_g T_k = T_k U_g = T_k$ on $L^2(\mu)$ for all $g\in\G$.
\end{enumerate}
\end{lemma}

\begin{proof}
\emph{(i)$\Rightarrow$(ii).} For any $f\in L^2(\mu)$,
\[
(U_g T_k f)(\x)
= (T_k f)(\act{g^{-1}}{\x})
= \int k(\act{g^{-1}}{\x},\x') f(\x')\,d\mu(\x').
\]
By invariance of $k$ in the first argument $U_g T_k = T_k$. Hence $T_k^* U_g^{*} = T_k^*$ and $T_k^*=T_k$ (self-adjoint) and $U_g^*=U_{g^{-1}}$ so $T_k U_{g^{-1}}= T_k$. This is true for all $g\in \G$ hence $U_g T_k = T_k U_g = T_k$. 

\emph{(ii)$\Rightarrow$(i).} For $\varphi,\psi\in L^2(\mu)$,
\[
\iint k(\x,\x')\,\varphi(\x)\psi(\x')\,d\mu(\x)d\mu(\x')
= \langle \varphi, T_k \psi\rangle
= \langle \varphi, T_k U_g \psi\rangle.
\]
Expanding the last inner product, we get by change of variable and invariance of $\mu$
\[
\iint k(\x,\x')\,\varphi(\x)\psi(\act{g^{-1}}{\x'})\,d\mu(\x)d\mu(\x')
= \iint k(\x,\act{g}{\x'})\,\varphi(\x)\psi(\x')\,d\mu(\x)d\mu(\x').
\]
Hence for all $\varphi,\psi$,
\(
\iint [k(\x,\x')-k(\x,\act{g}{\x'})]\,\varphi(\x)\psi(\x')\,d\mu(\x)d\mu(\x')=0,
\)
which implies $k(\x,\act{g}{\x'})=k(\x,\x')$ $\mu\otimes\mu$-a.e. Symmetry implies argumentwise $\G$-invariance.
\end{proof}

We now show that $U_g T = T$ is preserved if we apply a function $f$ such that $f(0)=0$ to the spectrum of $T$.

\begin{lemma}[Borel functional calculus preserves invariance]
\label{lem:func_calc_commute}
Let $T$ be a self-adjoint compact operator on a Hilbert space $\mathcal H$ with eigendecomposition $T=\sum_i \lambda_i \phi_i \otimes \phi_i$, and let $\{U_g\}_{g\in\G}$ be a unitary
representation such that $U_g T = T U_g = T$ for all $g\in\G$. For a bounded Borel function $f:\R\to\R$, define $f(T)=\sum_i f(\lambda_i) \phi_i\otimes \phi_i$. Then for such $f$ with $f(0)=0$, we have
\[
U_g f(T) \;=\; f(T)\,U_g \;=\; f(T) \qquad \text{for all } g\in\G.
\]
\end{lemma}

\begin{proof}
\textbf{Proof sketch:} The assumption $U_g T = T$ forces $U_g$ to act as the identity on each nonzero eigenspace of $T$, which directly yields $U_g f(T)=f(T)$ for any bounded Borel $f$ with $f(0)=0$.

\textbf{Step 1 (spectral decomposition for compact self-adjoint $T$ without measures).}
Since $T$ is compact and self-adjoint, its spectrum is $\sigma(T)=\{0\}\cup\{\lambda_n:n\in I\}$ where $I$ is finite or countable, each $\lambda_n\neq 0$ is an eigenvalue of finite multiplicity, and $\lambda_n\to 0$ if infinite. Let $E_{\lambda}$ denote the eigenspace for $\lambda\neq 0$, and let $E_0=\ker T$. We have the orthogonal decomposition
\[
\mathcal H \;=\; E_0 \;\oplus\; \bigoplus_{\lambda\in \sigma(T)\setminus\{0\}} E_\lambda,
\]
and $T$ acts as scalar multiplication on each $E_\lambda$:
$T|_{E_\lambda}=\lambda\,\mathrm{Id}_{E_\lambda}$, $T|_{E_0}=0$.
Let $P_\lambda$ be the orthogonal projector onto $E_\lambda$ (for $\lambda\neq 0$) and $P_0$ onto $E_0$.
Then for every $v\in\mathcal H$ with expansion $v=v_0+\sum_{\lambda\neq 0} v_\lambda$ ($v_\lambda:=P_\lambda v$), we have
\[
T v \;=\; \sum_{\lambda\neq 0} \lambda\, v_\lambda.
\]

\textbf{Step 2 ($U_g$ fixes each nonzero eigenspace pointwise).}
From $U_g T = T$ we get, for any $v\in E_\lambda$ with $\lambda\neq 0$,
\[
\lambda\, U_g v \;=\; U_g(Tv) \;=\; Tv \;=\; \lambda\, v,
\]
hence $U_g v = v$. Thus $U_g$ acts as the identity on each $E_\lambda$ ($\lambda\neq 0$).
Equivalently, $U_g P_\lambda = P_\lambda U_g = P_\lambda$ for all $\lambda\neq 0$.
(There is no restriction on $U_g$ inside $E_0=\ker T$.)

\textbf{Step 3 (defining $f(T)$ for bounded Borel $f$ with $f(0)=0$).}
Because $\sigma(T)\setminus\{0\}$ is at most countable and $T$ is diagonal on $\{E_\lambda\}$,
we can define $f(T)$ by applying $f$ on the spectrum of $T$ as 
\[
f(T)\,v \;:=\; \sum_{\lambda\in \sigma(T)\setminus\{0\}} f(\lambda)\, v_\lambda,
\qquad v=v_0+\sum_{\lambda\neq 0} v_\lambda,\; v_\lambda\in E_\lambda.
\]
The series converges in norm since the $E_\lambda$ are mutually orthogonal and
$\|f(T)v\|^2 = \sum_{\lambda\neq 0} |f(\lambda)|^2 \|v_\lambda\|^2 \le \big(\sup_{\lambda\neq 0}|f(\lambda)|^2\big)\sum_{\lambda\neq 0}\|v_\lambda\|^2 \le \|f\|_\infty^2 \|v\|^2$.
Thus $f(T)$ is a bounded operator with $\|f(T)\|\le \|f\|_\infty$.
(When $f(0)=0$, there is no contribution on $E_0$.)

\textbf{Step 4 (invariance and commutation).}
For $v=v_0+\sum_{\lambda\neq 0} v_\lambda$ as above and any $g\in\G$,
Step~2 gives $U_g v = U_g v_0 + \sum_{\lambda\neq 0} v_\lambda$ and $P_\lambda U_g = P_\lambda$ for $\lambda\neq 0$.
Hence
\[
U_g f(T)\,v
= U_g \Big( \sum_{\lambda\neq 0} f(\lambda)\, v_\lambda \Big)
= \sum_{\lambda\neq 0} f(\lambda)\, U_g v_\lambda
= \sum_{\lambda\neq 0} f(\lambda)\, v_\lambda
= f(T)\,v,
\]
i.e., $U_g f(T)=f(T)$. In particular $U_g f(T)=f(T)U_g=f(T)$ for all $g\in\G$.
\end{proof}

\paragraph{Consequence.}
If $k$ is $\G$-invariant, then so is $\kplus$ (\Cref{eq:kplus-spectral}).

\section{Eigendecay comparison}
\label{app:eigendecay}

In this appendix, we discuss in more details the empirical observations made in \Cref{sec:eigendecay} and formally derive some inequalities between Schatten norms of integral operators associated with $\kavg$ and $\kplus$.

\subsection{Empirical Observations}
\label{app:empirical-spectra}
Here, we further discuss the empirical spectra reported in 
\Cref{fig:merged} (middle and right columns).
% \Cref{fig:eigendecay}.

\paragraph{Computation of spectra.} 
The normalized Gram matrices $\K / n$ (where $\K = (k(\x_i,\x_j))_{1\leq i,j\leq n}$) reported in 
\Cref{fig:merged} 
% \Cref{fig:eigendecay} 
are computed from $n=3000$ i.i.d.\ samples $\x_i\in\Sp$. We compare the spectra obtained with $k\in\{\kb, \kavg, \kplus^{(\D)}\}$ with $\D=\{\x_1,\dots,\x_n\}$ and each $\x_i$ being chosen uniformly in $\Sp=[-1,1]$. We also report the spectrum of $\kb$ when observations $\x_i$ are instead sampled from an alternative domain $\Sp'$ of reduced volume, chosen such that $\vol(\Sp') = \vol(\Sp)/|\G|$. Finally, note that because $\D$ is a set of i.i.d.\ observations, the spectrum of $\kplus^{(\D)}$ approximates the one of $\kplus$ on $\Sp$ (see \Cref{app:kplus-approx}) so our observations transfer to $\kplus$.

\paragraph{$\kplus^{(\D)}$ on $\Sp$ vs.\ $\kb$ on $\Sp'$.}
For the base kernels $\kb$ and groups $\G$ considered, the spectrum of $\kplus^{(\D)}$ on $\Sp=[-1,1]$ exactly matches that of $\kb$ on the reduced domain $\Sp'$. This indicates that $\kplus^{(\D)}$ faithfully incorporates the extra similarities induced by $\G$-invariance: it retains the eigendecay of $\kb$, but as if it were defined on the quotient space $\Sp/\G$ of effective volume $\vol(\Sp)/|\G|$.\footnote{For a finite group $\G$ of isometries, one indeed has $\vol(\Sp/\G) = \vol(\Sp)/|\G|$~\citep{petersen2006riemannian}.}

\paragraph{$\kplus^{(\D)}$ on $\Sp$ vs.\ $\kavg$ on $\Sp$.}
From 
\Cref{fig:merged} (middle and right columns)
% \Cref{fig:eigendecay}
, it is clear that the spectrum of $\kavg$ decays at least as fast as that of $\kplus^{(\D)}$. They coincide for the RBF kernel and $\kavg$ decays even faster for the Matérn kernel. In principle, this suggests that $\kavg$ should admit tighter information-gain bounds and thus better regret guarantees. However, our empirical results contradict this prediction, as $\kplus^{(\D)}$ consistently outperforms $\kavg$. This discrepancy highlights the fact that eigendecay alone does not fully explain BO performance, as pointed out in \Cref{sec:eigendecay,sec:conclusion}.

\subsection{Schatten Norm inequalities}

While the empirical spectra in \Cref{app:empirical-spectra} already highlight a mismatch between eigendecay and observed BO performance, one may ask whether formal inequalities between the operators induced by $\kavg$ and $\kplus$ can be established. We record here for completeness that it is possible to control the Schatten class of $\kplus$ in terms of the one of $\kavg$. 

Assume: $(\Sp,\mu)$ is a probability space on which a finite group $\G$ acts measurably, and the base kernel $\kb$ is bounded, symmetric, PSD, and nonnegative. Define
\[
\kavg(\x,\x'):=\frac{1}{|\G|^2}\sum_{g,g'\in\G}\kb(\act{g}{\x},\act{g'}{\x'}),\qquad
\kmax(\x,\x'):=\max_{g,g'\in\G}\kb(\act{g}{\x},\act{g'}{\x'})
\]
and $\kplus$ as the kernel corresponding to the positive part of $T_{\kmax}$: $T_{\kplus} = (T_{\kmax})_+$.

\paragraph{Schatten norm interpolation.}
Let $H=L^2(\mu)$ be the separable Hilbert space of squared integrable functions on $(\Sp,\mu)$, $T:H\to H$ a compact operator, and write
$s_i(T)$ for the singular values of $T$, i.e.\ $s_i(T) = \sqrt{\lambda_i(T^*T)}$, 
arranged in nonincreasing order and counted with multiplicity. 
The Schatten-$p$ norm is defined as
\[
  \|T\|_{S_p} := \Big(\sum_i s_i(T)^p\Big)^{1/p}, \qquad 1\le p < \infty,
  \qquad \|T\|_{S_\infty} := \sup_i s_i(T).
\]

\begin{lemma}[Monotonicity for pointwise kernels]
\label{lem:kernel-monotonicity}
If two kernels $k,k'$ are bounded and satisfy $0\le k\le k'$ pointwise, then $\|T_k\|_{S_p} \leq \|T_{k'}\|_{S_p}$ for $p=2,\infty$. If $k$ and $k'$ are also PSD, then $\|T_k\|_{S_p} \leq \|T_{k'}\|_{S_p}$ for $p=1$ too.
\end{lemma}

\begin{proof}
For $p=\infty$, the Schatten $p$-norm is the operator norm $\|T\|_{\op}=\sup_{\|f\|_H=1}\|Tf\|_H$. 
Pointwise $0\le k\le k'$ implies 
\(\|T_k f\|_H \le \|T_{k'}|f|\|_H \le \|T_{k'}\|_{S_\infty}\|f\|_H\), 
so taking the supremum over $\|f\|_H=1$ yields 
\(\|T_k\|_{S_\infty}\le \|T_{k'}\|_{S_\infty}\).
If $T=T_k$ is the integral operator associated with a nonnegative kernel $k$, then $\|T_k\|_{S_2} = \|k\|_{L^2(\mu\otimes \mu)}$. 
Hence pointwise $0\le k\le k'$ gives 
$\|T_k\|_{S_2} \le \|T_{k'}\|_{S_2}$ for $p=2$ as well. Finally when $k$ is PSD, we have $\|T_k\|_{S_2} = \int_x k(x,x) d\mu(x)$ (and similarly for $k'$) and again a pointwise comparison yields the result.
\end{proof}

From this we immediately obtain, for our specific kernels that for $p=2,\infty$, and also $p=1$ if $\kmax$ is PSD:
\[
  k_{\mathrm{avg}} \;\le\; k_{\mathrm{max}} 
  \;\le\; |\mathcal G|^2\,k_{\mathrm{avg}}
  \quad\Rightarrow\quad
  \|T_{k_{\mathrm{avg}}}\|_{S_p} \;\le\;
  \|T_{k_{\mathrm{max}}}\|_{S_p} \;\le\;
  |\mathcal G|^2\,\|T_{k_{\mathrm{avg}}}\|_{S_p}
\]

\begin{lemma}[Interpolation inequalities for Schatten norms]
\label{lem:schatten-interp}
For any nonnegative sequence $a=(a_i)_{i\ge1}$ one has
\[
  \|a\|_{\ell^p} \;\le\; \|a\|_{\ell^2}^{\,2/p}\,
    \|a\|_{\ell^\infty}^{\,1-2/p} \qquad (p\ge2),
\]
\[
  \|a\|_{\ell^p}^p \;\le\; 
    \|a\|_{\ell^1}^{\,2-p}\,
    \|a\|_{\ell^2}^{\,2(p-1)} \qquad (1\le p\le2).
\]
\end{lemma}

\begin{proof}
For $p\ge2$, 
\(\sum_i a_i^p = \sum_i a_i^{p-2} a_i^2 
\le \|a\|_{\ell^\infty}^{p-2}\sum_i a_i^2\),
giving the stated inequality.
For $1\le p\le2$, write
\[
\sum_i a_i^p=\sum_i a_i^{\,2-p}\,a_i^{\,2(p-1)}.
\]
Let $r=\frac{1}{2-p}$ and $s=\frac{1}{p-1}$ (with the usual convention $1/0=\infty$). 
For $1<p<2$ we have $1<r,s<\infty$ and by Hölder,
\[
\sum_i a_i^p
\;\le\;
\Bigl(\sum_i (a_i^{2-p})^r\Bigr)^{1/r}
\Bigl(\sum_i (a_i^{2(p-1)})^s\Bigr)^{1/s}
=
\bigl(\sum_i a_i\bigr)^{1/r}
\bigl(\sum_i a_i^2\bigr)^{1/s}.
\]
Since $1/r=2-p$ and $1/s=p-1$, this gives
\[
\|a\|_{\ell^p}^p
\;\le\;
\|a\|_{\ell^1}^{\,2-p}\,
\|a\|_{\ell^2}^{\,2(p-1)}.
\]
The endpoint cases $p=1,2$ follow by continuity (and are trivial directly). 
\end{proof}

Applied to $a_i=s_i(T)$, Lemma~\ref{lem:schatten-interp} yields the standard Schatten interpolation inequalities:
\[
  \|T\|_{S_p}\ \le\ \|T\|_{S_2}^{\,2/p}\,
    \|T\|_{S_\infty}^{\,1-2/p}, \quad (p\ge2),
\]
\[
  \|T\|_{S_p}\ \le\ 
  \big(\|T\|_{S_1}\big)^{\frac{2}{p}-1}\,
  \big(\|T\|_{S_2}^2\big)^{1-\frac{1}{p}}, \quad (1\le p\le2).
\]

Since the spectrum of $T_{\kplus}$ is the positive part of the one of $T_{\kmax}$, we have $\|T_{\kplus}\|_{S_p} \leq \|T_{\kmax}\|_{S_p}$. We deduce the next lemma.

\begin{lemma}
    For $p\geq 2$:
    \[
    \|T_{\kplus}\|_{S_p} \leq \|T_{\kmax}\|_{S_p}\ \le |\G| \|T_{\kavg}\|_{S_2}^{\,2/p}\,
    \|T_{\kavg}\|_{S_\infty}^{\,1-2/p}
    \]
    and if $\kmax$ is already PSD then for $1\leq p\leq 2$:
    \[
    \|T_{\kplus}\|_{S_p}=\|T_{\kmax}\|_{S_p}\ \le |\G| \big(\|T_{\kavg}\|_{S_1}\big)^{2/p-1}\,\big(\|T_{\kavg}\|_{S_2}^2\big)^{1-1/p}
    \]
    and 
    \[
    \|T_{\kavg}\|_{S_p} \le \big(\|T_{\kmax}\|_{S_1}\big)^{2/p-1}\,\big(\|T_{\kmax}\|_{S_2}^2\big)^{1-1/p}.
    \]
\end{lemma}

\section{Benchmarks} \label{app:benchmarks}

In this appendix, we describe the experimental setting and all the benchmarks used to produce the numerical results of Section~\ref{sec:xps}.

\subsection{Experimental Details}

In our experiments, every BO algorithm is implemented with the same BO library, namely BOTorch~\citep{balandat2020botorch}. All of them are initialized with five observations sampled uniformly in $\Sp$. After that, at each iteration $t$, every BO algorithm must:
\begin{itemize}
    \item \textbf{Fit its kernel hyperparameters.} This is done by gradient ascent of the Gaussian likelihood, as recommended by BOTorch. The hyperparameters are the signal variance $\lambda$, the lengthscale $l$ and the observational noise level $\sigma^2_0$.
    \item \textbf{Optimize GP-UCB to find $\x_t$.} This is done by multi-start gradient ascent, using the \texttt{optimize\_acqf} function from BOTorch. As values of $\beta_t$ recommended by~\citet{srinivas-information:2012} turn out to be too exploratory in practice, we set $\beta_t = 0.5 d \log(t)$.
    \item \textbf{Observe $y(\x_t) = f(\x_t) + \epsilon_t$.} Function values are corrupted by noise whose variance is $2\%$ of the signal variance.
\end{itemize}

We optimize over 50 iterations and typically measure the cumulated regret along the optimizer's trajectory.

All experiments are replicated across ten independent seeds and are run on a laptop equipped with an Intel Core i9-9980HK @ 2.40 GHz with 8 cores (16 threads). No graphics card was used to speed up GP inference. The typical time for each maximization problem ranged from $\sim$1 minute (two-dimensional Ackley, $|\G| = 8$) to $\sim$15 minutes (five-dimensional Rastrigin, $|\G| = 3840$).

\subsection{Benchmarks}

As GP-UCB is naturally formulated for maximization tasks, all benchmarks that require minimization have been multiplied by $-1$ to produce benchmarks on maximization.

\paragraph{Ackley.} The $d$-dimensional Ackley function is defined on $\mathcal{S} = [-16, 16]^d$ with a global minimum at $f_{\text{Ackley}}(\bm 0) = 0$ and has the following expression:
\begin{equation}\label{def:ackley}
f_{\text{Ackley}}(\x) = -a \exp\left(-b \sqrt{\frac{1}{d} \sum_{i = 1}^d x_i^2}\right) - \exp\left(\frac{1}{d} \sum_{i = 1}^d \cos(cx_i)\right) + a + \exp(1),
\end{equation}
where we set $a = 20$, $b = 0.2$ and $c = 2\pi$ as recommended.

The $d$-dimensional Ackley is invariant to the hyperoctahedral group in $d$ dimensions, which includes permutations composed with sign-flips. Consequently, in $d$ dimensions, $|\G| = \underbrace{2^d}_{\text{sign flips}} \underbrace{d!}_{\text{permutations}}$.

\paragraph{Griewank.} The $d$-dimensional Griewank function is defined on $\Sp = [-600, 600]^d$ with a global minimum at $f_{\text{Griewank}}(\bm 0) = 0$ and has the following expression:
\[
    f_{\text{Griewank}}(\bm x) = \sum_{i = 1}^d \frac{x_i^2}{4000} - \prod_{i = 1}^d \cos\left(\frac{x_i}{\sqrt{i}}\right) + 1.
\]

The $d$-dimensional Griewank is invariant to sign-flips of all $d$ coordinates. Therefore, in $d$ dimensions, $|\G| = 2^d$.

\paragraph{Rastrigin.} The $d$-dimensional Rastrigin function is defined on $\Sp = [-5.12, 5.12]^d$ with a global minimum at $f_{\text{Rastrigin}}(\bm 0) = 0$ and has the following expression:
\[
    f_{\text{Rastrigin}}(\bm x) = 10d + \sum_{i = 1}^d \left(x_i^2 - 10 \cos\left(2 \pi x_i\right)\right).
\]

The $d$-dimensional Rastrigin is invariant to the hyperoctahedral group in $d$ dimensions, which includes permutations composed with sign-flips. Consequently, in $d$ dimensions, $|\G| = \underbrace{2^d}_{\text{sign flips}} \underbrace{d!}_{\text{permutations}}$.

\paragraph{Radial.} Our radial benchmark is defined on $\Sp = [-10, 10]^2$ with a global minimum at $f_{\text{Radial}}(\x^*) = 0$, where $\x^*$ is any $\x \in \Sp$ such that $||\x||_2 = ab$. It has the following expression:
\begin{equation}\label{def:radial}
    f_{\text{Radial}}(\x) = f_{\text{Rastrigin}}\left(\frac{||\x||_2}{a} - b\right)
\end{equation}
where we set $a = 10 \sqrt{2}$, $b = 0.8$ and where $f_{\text{Rastrigin}}$ is the one-dimensional Rastrigin benchmark.

Our radial benchmark is invariant to planar rotations. Consequently, $\G$ comprises an uncountably infinite number of symmetries.

\paragraph{Scaling.} Our scaling benchmark is defined on $\Sp = [0.1, 10]^2$ with a global minimum at $f_{\text{Scaling}}(\x^*) = 0$, where $\x^*$ is any $\x = (x_1, x_2) \in \Sp$ such that $x_1 = x_2$. It has the following expression:
\[
    f_{\text{Scaling}}(\x) = \left(\frac{x_1}{x_2} - 1\right)^2.
\]

Our scaling benchmark is invariant to rescaling of both coordinates. Consequently, $\G$ comprises an uncountably infinite number of symmetries.

\paragraph{WLAN.} The WLAN benchmark consists of $p$ users scattered in an area $\mathcal{A} = [-50, 50]^2$ and $m$ access points (APs) to be placed in $\mathcal{A}$. Therefore, the search space is $\Sp = \mathcal{A}^m$, which is $2m$-dimensional. Given a placement $\left\{(x_i, y_i)\right\}_{i \in [m]}$ for the $m$ APs, each user is allocated to the AP closest to it. Each AP $i$ has now a set $\mathcal{U}(x_i, y_i)$ of associated users.

Assume the AP $i$ is associated to the user $j \in \mathcal{U}(x_i, y_i)$. The quality of service (QoS) of the AP-user association is given by the Shannon capacity in Mbps:
\[
    C_{ij} = W \log_2(1 + \gamma_{ij}),
\]
where $W$ is the bandwidth of the signal (in MHz), $\gamma_{ij}$ is the signal to interference-plus-noise ratio (SINR) defined by
\[
    \gamma_{ij} = \frac{P_{ij}}{N + \sum_{k \neq i}^m P_{kj}}.
\]
Here, $N$ is the background noise (in mW) while $P_{ij}$ is the power (in mW) received by user $j$ from AP $i$. The power received is computed using the well-known log-distance path-loss:
\[
    P_{ij} = 10^{-L/10} \min(d_{ij}^{-\lambda}, 1).
\]
where $d_{ij}$ is the Euclidean distance between AP $i$ and user $j$ and where $L$ and $\lambda$ are positive constants.

Finally, the objective function to maximize is the cumulated sum of Shannon capacities for every AP-user association:
\[
    f_{\text{WLAN}}(\x, \y) = \sum_{i = 1}^m \sum_{j \in \mathcal{U}(x_i, y_i)} C_{ij},
\]
where $\x = (x_1, \cdots, x_m)$ and $\y = (y_1, \cdots, y_m)$ are the positions of the $m$ APs on the x-axis and y-axis, respectively. In our experiment, we set $W = 1$~MHz, $L = 46.67$~dBm, $\lambda = 3$, $N = -85$~dBm, $m = 4$~APs and $p = 16$~users.

Because the main goal of this task is to optimize the QoS by placing a set of APs, our objective $f_{\text{WLAN}}$ is invariant to any permutation of the coordinates in the vectors $\x$ and $\y$. Therefore, $|\G| = m!$.

Figure~\ref{fig:wlan} shows a depiction of the WLAN and the best placement of APs inferred by GP-UCB using $\kplus^{(\D)}$.

\begin{figure}
    \centering
    \includegraphics[height=5cm]{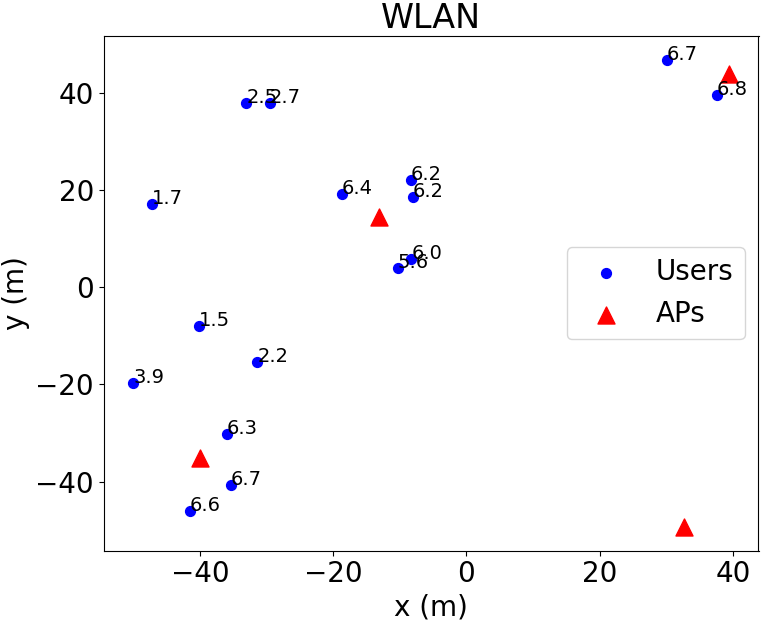}
    \caption{WN with the best positions of APs found by GP-UCB with $\kplus^{(\D)}$. APs are depicted by red triangles and users with blue circles. The throughput for each user is shown in Mbps.}
    \label{fig:wlan}
\end{figure}

\end{document}

%% file: math_commands.tex
\usepackage{amsmath,amsfonts,bm}

\def\1{\bm{1}}

\DeclareMathAlphabet{\mathsfit}{\encodingdefault}{\sfdefault}{m}{sl}
\SetMathAlphabet{\mathsfit}{bold}{\encodingdefault}{\sfdefault}{bx}{n}

\newcommand{\E}{\mathbb{E}}

\newcommand{\R}{\mathbb{R}}

\newcommand{\Var}{\mathrm{Var}}

\newcommand{\Cov}{\mathrm{Cov}}
\DeclareMathOperator*{\argmax}{arg\,max}
\DeclareMathOperator*{\argmin}{arg\,min}